\documentclass{article}

\usepackage{microtype}
\usepackage{graphicx}
\usepackage{booktabs} 
\usepackage{stackengine}
\usepackage{hyperref}

\usepackage[symbol]{footmisc}

 \usepackage[accepted]{icml2024}

\usepackage{amsmath}
\usepackage{amssymb}
\usepackage{mathtools}
\usepackage{amsthm}
\usepackage{enumitem}
\usepackage{subcaption} 
\usepackage{multirow} 
\usepackage[capitalize,noabbrev]{cleveref}

\usepackage{verbatim}

\theoremstyle{plain}
\newtheorem{theorem}{Theorem}[section]

\newtheorem{claim}[theorem]{Claim}

\theoremstyle{definition}

\theoremstyle{remark}

\usepackage{tikz}
\usetikzlibrary{positioning}
\newdimen\nodeDist
\usepackage[textsize=tiny]{todonotes}

\icmltitlerunning{Trained RFs Completely Reveal your Dataset}

\begin{document}

\def\nexamples{N}
\def\nattributes{M}
\def\attributes[#1]{\mathbf{x}_{#1}}
\def\classes{\mathcal{C}}
\def\class{c}
\def\oheclass[#1,#2]{z_{#2#1}}
\def\example{k}
\def\feature{i}
\def\singleattribute[#1]{f_{#1}} 
\def\classset[#1]{Z_{#1}}
\def\val{b}
\def\values{\mathcal{B}}
\def\maxval{7}
\def\attrvalue{a} 
\def\splitvalues[#1]{\mathcal{A}_{#1}}
\def\intervalvalues[#1]{\mathcal{I}_{#1}}

\def\lexicorder{\leq_L}
\def\ohevects{vects}
\def\ohegroup{w}

\def\forest{\mathcal{T}}
\def\tree{t}
\def\depth{d}
\def\node{v}

\def\internalnodes[#1]{\mathcal{V}_{#1}^{I}}
\def\leaves[#1]{\mathcal{V}_{#1}^{L}}
\def\leftchild[#1]{l(#1)}
\def\rightchild[#1]{r(#1)}
\def\nodesupport[#1,#2,#3]{n_{#1#2#3}} 
\def\nodesforfeat[#1,#2]{\mathcal{V}_{#1#2}^{I}}
\def\positivesplits[#1,#2]{\Phi_{#2}^+} 
\def\negativesplits[#1,#2]{\Phi_{#2}^-} 
\def\varlambda[#1,#2,#3]{\lambda_{#1#2#3}}
\def\vary[#1,#2,#3]{y_{#1#2#3}}
\def\varyb[#1,#2,#3,#4]{y_{#1#2#3#4}}
\def\varx[#1,#2]{x_{#1#2}}
\def\varz[#1,#2]{z_{#1#2}}
\def\varq[#1,#2,#3]{q_{#1#2#3}}
\def\varqalt[#1,#2,#3,#4]{q_{#1#2#3#4}}
\def\attrib[#1]{f_{#1}}
\def\lb[#1]{l_{#1}}
\def\ub[#1]{u_{#1}}
\def\root[#1]{r_{#1}}
\def\parent[#1]{f_{#1}}
\def\treepath[#1,#2]{P^{#1}_{#2}}
\def\fixedatt[#1]{\phi(#1)}

\def\Z{\mathbb{Z}}
\def\np{\mathcal{NP}}

\def\sklearn{\texttt{scikit-learn}}
\def\ortools{\texttt{OR-Tools}}
\def\gurobi{\texttt{Gurobi}}
\def\jul#1{\textcolor{orange}{#1}}

\def\probname{DRP}
\def\mlprobname{MLDRP}

\def\scipy{\texttt{Scipy}}
\def\draft{DRAFT}

\def\clauses{L}
\def\clause{l}
\def\nsatvars{V}

\twocolumn[
\icmltitle{Trained Random Forests Completely Reveal your Dataset}


\icmlsetsymbol{equal}{*}

\begin{icmlauthorlist}
\icmlauthor{Julien Ferry}{xxx}
\icmlauthor{Ricardo Fukasawa}{yyy}
\icmlauthor{Timothée Pascal}{zzz}
\icmlauthor{Thibaut Vidal}{xxx}
\end{icmlauthorlist}

\icmlaffiliation{xxx}{CIRRELT \& SCALE-AI Chair in Data-Driven Supply Chains, Department of Mathematics and Industrial Engineering, Polytechnique Montréal, Canada}
\icmlaffiliation{yyy}{Department of Combinatorics and Optimization, University of Waterloo, Canada}
\icmlaffiliation{zzz}{Ecole nationale des ponts et chaussées, Paris, France}

\icmlcorrespondingauthor{Thibaut Vidal}{thibaut.vidal@polymtl.ca}

\icmlkeywords{Machine Learning, ICML}

\vskip 0.3in
]



\printAffiliationsAndNotice{}  

\begin{abstract}
We introduce an optimization-based reconstruction attack capable of completely or near-completely reconstructing a dataset utilized for training a random forest. Notably, our approach relies solely on information readily available in commonly used libraries such as \sklearn{}. To achieve this, we formulate the reconstruction problem as a combinatorial problem under a maximum likelihood objective. We demonstrate that this problem is $\np$-hard, though solvable at scale using constraint programming -- an approach rooted in constraint propagation and solution-domain reduction. Through an extensive computational investigation, we demonstrate that random forests trained without bootstrap aggregation but with feature randomization are susceptible to a complete reconstruction. This holds true even with a small number of trees. Even with bootstrap aggregation, the majority of the data can also be reconstructed. These findings underscore a critical vulnerability inherent in widely adopted ensemble methods, warranting attention and mitigation. Although the potential for such reconstruction attacks has been discussed in privacy research, our study provides clear empirical evidence of their practicability.
\end{abstract}

\section{Introduction}
\label{introduction}

Machine learning (ML) techniques are increasingly used on sensitive data, such as medical records for kidney exchange~\cite{DBLP:conf/aaai/0001CDM21}, criminal records~\cite{angwin2016machine} or credit history. As this raises significant ethical and societal challenges, the use of such private data is directly regulated by several legal texts, such as the recent European Union General Data Protection Regulation\footnote[1]{\url{https://gdpr-info.eu/}} 
or the forthcoming AI Act\footnote[7]{\url{https://artificialintelligenceact.eu/}}.
Privacy has attracted significant attention during the last decades~\citep{liu2021machine} in order to protect sensitive or personal information about individual users while still being able to extract useful patterns from data. 
Moreover, privacy risks may further be exacerbated by the consideration of other ethical desiderata, e.g., when releasing a trained ML model for the sake of transparency.

In this work, we specifically study such privacy concerns in the white-box setting in which a trained random forest (RF) is publicly released. More precisely, we attempt to reconstruct the entire dataset used to train the RF by only using information available by default in widespread libraries such as \sklearn{}~\cite{scikit-learn}, namely the structure of the trees within the forest and the class cardinalities provided within each node.

While reconstruction attacks have been previously studied~\citep{doi:10.1146/annurev-statistics-060116-054123}, to the best of our knowledge, no work could consistently reconstruct an entire dataset from a trained RF. While some information can be extracted from single trees regarding the number of examples with specific combinations of features, the path taken by each individual example in each tree is unknown. Consequently, it is challenging to combine the information provided by different trees to effectively narrow down the potential datasets. To achieve this goal, we formalize the \emph{maximum-likelihood dataset reconstruction problem} and formulate it as a unified Constraint Programming (CP) model over the forest. With this, we can leverage the solution capabilities of modern CP algorithms based on constraint propagation, solution domain reduction, exploration, and backtracking. In an extensive computational campaign, we show that our methodology achieves nearly flawless recovery for RFs trained without bootstrap aggregation but with feature randomization. Even in cases where bootstrap aggregation is employed, our approach successfully recovers the majority of the data. In summary, the main contributions of this study are:

\begin{itemize}[nosep]
    \item A formalization of the \emph{maximum-likelihood dataset reconstruction problem} for random forests
    \item A proof of $\np$-hardness for this problem. This is, however, a limited safeguard since the relentless progress of generalist combinatorial optimization algorithms (i.e., based on CP or mixed-integer programming) permits solving many $\np$-hard problems at scale nowadays.
    \item The proposal of a CP formulation amenable to an efficient solution using state-of-the-art algorithms.
    \item Extensive computational experiments demonstrating how even a reasonably small number of trees reveal the quasi-totality of the datasets on standard applications. Our source code is openly accessible at \url{https://github.com/vidalt/DRAFT} in the form of a user-friendly Python module named \draft{} \emph{(Dataset Reconstruction Attack From Trained ensembles)}, under a MIT license. 

\end{itemize}

\section{Technical Background}
\label{sec:technical_background}

\paragraph{Supervised Machine Learning (ML).} Let ${\{\attributes[\example];\class_{\example}\}}^{\nexamples}_{\example=1}$ be a training set in which each example $\example$ is characterized by a vector $\attributes[\example] \in \{0,1\}^{\nattributes}$ of $\nattributes$ binary attributes and a class $\class_{\example} \in \classes$. We let $\oheclass[\class,\example]$ be a one-hot encoding of the classes, which is $1$ if $\class_{\example}=\class$, and $0$ otherwise.
Moreover, in some situations, several binary features are used to one-hot encode a single original numerical or categorical attribute. In such case, precisely one of these binary features is $1$, and the others are $0$. We let $\ohevects$ be the list of the different groups (if any) of binary attributes one-hot encoding the same original feature.

\paragraph{Random Forests (RFs).} The training dataset is used to build a random forest $\forest$ in which each tree $\tree \in \forest$ is made of a set of internal nodes $\internalnodes[\tree]$ and a set of leaves $\leaves[\tree]$. 
Each internal node $\node \in \internalnodes[\tree]$ corresponds to a binary condition over the value of a given attribute. If the condition is satisfied, the example being classified descends towards the left child $\leftchild[\node]$ of the node, otherwise it descends towards its right child $\rightchild[\node]$. Once the example reaches a leaf $\node \in \leaves[\tree]$ (terminal node), it is classified according to the class associated with this leaf. Such class corresponds to the majority class among the training examples captured by the leaf. To compute it (and eventually assign class probabilities), each leaf contains the per-class number of training examples it captures. In popular ML libraries such as \texttt{scikit-learn}, such counts are also provided in the internal nodes, as shown in Figure~\ref{fig:example_toy_dt}. Then, for every node $\node \in \internalnodes[\tree] \bigcup \leaves[\tree]$, let $\nodesupport[\tree,\node,\class]$ denote the number of training examples of class~$\class$ that went through $\node$.

\paragraph{Training RFs.} To encourage diversity between the different trees within an RF, several randomization mechanisms are used during training. For instance, when building each individual tree, only a random subset of the $\nattributes$ features is considered to determine the best split at each node. Note that this mechanism is used in all our experiments, although we do not explicitly leverage it. Bootstrap aggregation (\emph{bagging}) is another popular and successful mechanism in RF training \citep{10.5555/2381019}. It consists in building $\lvert \forest \rvert$ separate training sets, one for each tree, by performing random sampling with replacement from the original training set ${\{\attributes[\example];\class_{\example}\}}^{\nexamples}_{\example=1}$. In consequence, not all examples of the original dataset are used for training each tree, while some appear multiple times. Algorithmic implementations for learning RFs are available within popular libraries such as \sklearn{}. While bagging is not mandatory, it is often used by default, as it lowers variance and enhances generalization. 
Finally, some support or size constraints are often set when training each tree. In particular, it is possible to set a maximum depth constraint ensuring that each tree has depth at most $\depth_{max}$.\\
In our framework, we leverage both the structure of the trees within the forest and the counts provided within each node to conduct a dataset reconstruction attack. We additionally take advantage of the theoretical probability distributions of the number of occurrences of each example within each tree's training set.

\paragraph{Constraint Programming (CP).} CP is a generic approach to finding feasible or optimal solutions to a wide variety of problems, including $\np$-hard ones. The basic principle is to define a set of \emph{decision variables} -- each allowed to take values within a given (discrete) domain -- and \emph{constraints} that express relationships between variables. Optionally, an \emph{objective function} may be provided to be maximized or minimized.
The types of allowed constraints depend on which specific CP solver is used, but linear and logical/implication constraints are typical examples.\\
CP solvers then combine several techniques (constraint propagation, backtracking, local search) to efficiently explore the search space and find a feasible/optimal solution. An overview of fundamental ideas/techniques in CP can be found in \citet{rossi2008constraint}. While solving CP models is theoretically $\np$-hard, state-of-the-art solvers can handle very large-scale problems in practice, and the performance of state-of-the-art solvers has dramatically increased.

\section{Related Works}
\label{sec:rel_works}

\begin{figure*}[h!t]
  \centering
  \begin{subfigure}[t]{0.48\textwidth}
    \centering\includegraphics[width=0.93\textwidth]{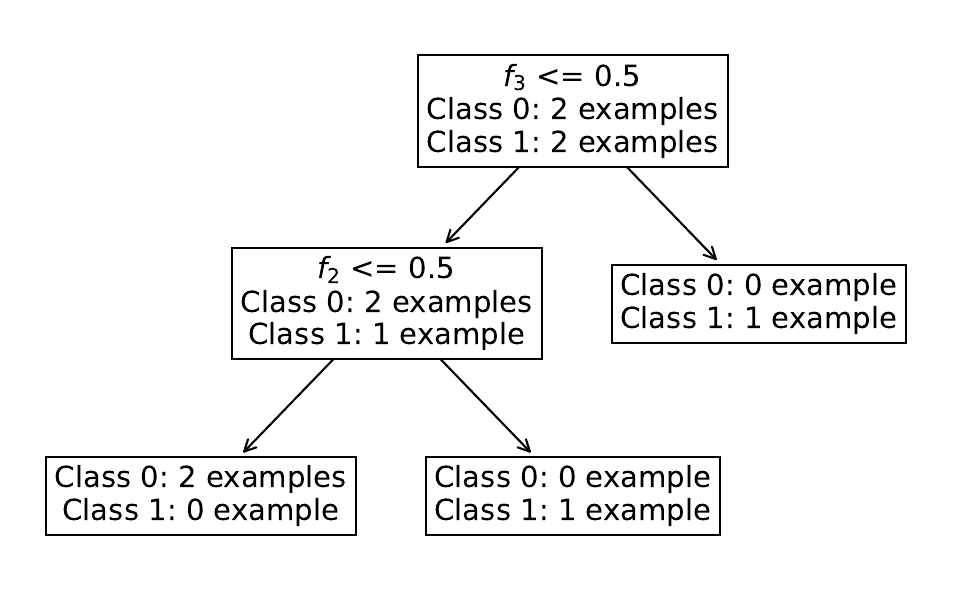}
    \caption{Tree $\tree_1$ (trained without using attribute $\singleattribute[1]$)}\label{fig:example_toy_dt_1}
  \end{subfigure}
  \begin{subfigure}[t]{0.48\textwidth}
    \centering\includegraphics[width=0.93\textwidth]{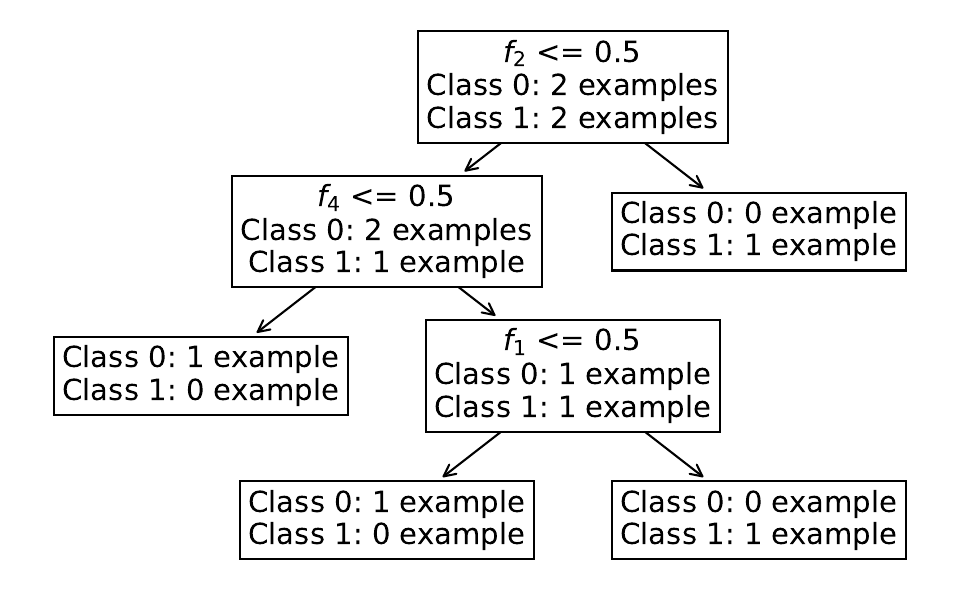}
    \caption{Tree $\tree_2$ (trained without using attribute $\singleattribute[3]$)}\label{fig:example_toy_dt_2}
  \end{subfigure}
    \caption{Example decision trees trained using \texttt{scikit-learn} on a small dataset (Table~\ref{tab:toy_dataset}).}
    \label{fig:example_toy_dt}
\end{figure*}

ML methods often exploit private data during training.
Consequently, it is crucial to ensure that their outputs -- which may be released directly or accessed through a dedicated API -- do not leak information regarding their inputs~\citep{DBLP:conf/pods/DinurN03}. 
\emph{Inference attacks} against ML models~\citep{DBLP:journals/corr/abs-2007-07646} precisely aim at exploiting the output of a learning algorithm 
to infer information regarding the training dataset.
Different attacks can be distinguished, depending on their specific objective. For instance, \emph{membership inference attacks} aim to infer whether an example was part of a model's training data or not~\citep{DBLP:conf/sp/ShokriSSS17,DBLP:conf/sp/CarliniCN0TT22}. 
In this work, we are interested in \emph{dataset reconstruction attacks}, which aim at reconstructing (entirely or partially) a model's training dataset~\citep{doi:10.1146/annurev-statistics-060116-054123}. The considered \emph{attack model} is as follows. We specifically target RF models, and consider the \emph{white-box setup}, in which the adversary has complete knowledge of the model's parameters instead of a black-box API access to it~\citep{DBLP:journals/corr/abs-2005-08679}. 
He also knows the domains of the different attributes involved in the data. Importantly, he does not intervene during the training process, but rather gets the trained model afterwards.

\textsc{Reconstruction Attack.} \emph{Given a trained RF, find a reconstructed version of its training set -- a value for each feature of each example -- that is feasible and likely w.r.t. the training process. Ideally, the reconstructed dataset should closely match the actual training data ${\{\attributes[\example];\class_{\example}\}}^{\nexamples}_{\example=1}$.} 

Reconstruction attacks are one of the most ambitious inference attacks against ML models, as they directly aim to recover entire parts of the training data. However, instead of attempting to reconstruct the whole training set, most reconstruction attacks only target retrieving part of it. For instance, the first reconstruction attacks (originally proposed against database access mechanisms) only aimed at retrieving one private binary attribute for all the database examples - assuming all other attributes were publicly known~\citep{DBLP:conf/pods/DinurN03,10.1145/1250790.1250804,doi:10.1146/annurev-statistics-060116-054123}. Some other studies only target reconstructing part of one particular example, given some public information about it~\citep{DBLP:conf/uss/FredriksonLJLPR14,DBLP:conf/ccs/FredriksonJR15}.
Other approaches require additional knowledge, such as intermediate gradients computed during collaborative~\citep{DBLP:conf/atis/PhongA0WM17} or online~\citep{DBLP:conf/uss/0001B0F020} learning, stationary points reached by gradient descent algorithms~\citep{DBLP:conf/nips/HaimVYSI22} or information regarding the model's fairness~\citep{hu2020inference,aalmoes2022dikaios,hamman2022can,ferry2023exploiting}. 

The most closely related works are those of~\citet{DBLP:conf/dbsec/GambsGH12} and~\citet{ferry:hal-04189566}. More precisely, \citet{DBLP:conf/dbsec/GambsGH12} showed that the structure of a single trained decision tree can be leveraged to build a probabilistic dataset encoding the whole set of reconstructions of the training data that are compatible with the provided tree's structure. This approach was later generalized in \citet{ferry:hal-04189566} to consider other simple interpretable models. Compared to previous studies, one of the key challenges addressed by our approach is to combine the information provided by several trees to achieve a feasible and accurate reconstruction. This is especially difficult since the number of occurrences (due to \emph{bagging}) and the path taken by each individual example in each tree is unknown. Consequently, we specifically design our method to handle the random selection of examples within each tree and formulate a maximum log-likelihood objective to guide the search.

\section{Illustrative Example}
\label{sec_illustrative_example}

\begin{table}[t]
\centering
\caption{Example binary dataset with $\nexamples = 4$ and $\nattributes = 4$.}\label{tab:toy_dataset}
\scalebox{0.85}
{
\begin{tabular}{@{}ccccc@{}}
\toprule
$\singleattribute[1]$ & $\singleattribute[2]$        & $\singleattribute[3]$ & $\singleattribute[4]$ & $\class$            \\ \midrule
0   & 0 & 0   & 1   & 0 \\
1            & 0 & 0   & 0   & 0 \\
0            & 1 & 0   & 0            & 1 \\
1   & 0          & 1   & 1   & 1 \\ \bottomrule
\end{tabular}
}
\end{table}

We first give an intuition of the reconstruction problem on a small dataset (Table~\ref{tab:toy_dataset}) with $4$ examples described by $4$ binary attributes $\singleattribute[\feature \in \{1..4\}]$ and a binary class $\class$.
Figure~\ref{fig:example_toy_dt} provides two decision trees trained on this dataset. 
Tree $\tree_1$ was trained without using $\singleattribute[1]$, while tree $\tree_2$ was trained without using $\singleattribute[3]$ (though this information is unknown to the reconstruction algorithm). For presentation simplicity, bagging is not used here, and therefore each training example is used a single time in each tree.
By following the paths from the root to each leaf within $\tree_1$ (Figure~\ref{fig:example_toy_dt_1}), one can set the value of some attributes within the reconstructed dataset by leveraging the performed splits and the per-node cardinalities. For instance, following the leftmost path, we can observe that the two examples belonging to class $0$ have value $0$ for both $\singleattribute[3]$ and $\singleattribute[2]$. Such information permits to fix some attributes' values directly. Similarly, according to $\tree_2$, there exists exactly one example of class $1$ with value $0$ for $\singleattribute[2]$, and another one with value $1$.

The main issue with such ad-hoc reasoning is that, except in some obvious cases (i.e., when \emph{all} examples of a certain class respect a given splitting condition), splits will permit quantifying \emph{how many} examples respect a certain condition without telling \emph{which} are these examples. Therefore, the biggest challenge of dataset reconstruction is to individually link the examples between the different trees and find a compatible dataset that respects all the cardinality constraints. This challenge is exacerbated by the bagging process, as in this case, the cardinalities within the trees' nodes may count some examples several times (and ignore some others).

\section{NP-Hardness Result}
\label{hardness}

In this section, we formally define the \emph{dataset reconstruction problem (\probname)} and show its $\np$-completeness. 
Part of the input data for (\probname) has already been defined before:
the set of classes $\classes$, the number $\nexamples$ of examples, the number $\nattributes$ of binary attributes, and the forest $\forest$, where each $\tree\in \forest$ is a binary tree.
Also, for each class $\class \in \classes$, tree $\tree\in \forest$ and each $\node\in \internalnodes[\tree] \cup \leaves[\tree]$, we are given an amount $\nodesupport[\tree,\node,\class]\in \Z_+$ of examples of class $\class$ that are
classified in node $\node$ of tree $\tree$.
 
In addition, we are given as input:
\begin{itemize}[nosep]
\item for every $\node\in \internalnodes[\tree]$, an attribute $\attrib[\node]\in \{1..\nattributes\}$.
\item a set $\values$ of integer values, that represent how many times a sample may appear in a tree
\end{itemize}

We assume that the data satisfies the following properties:
\begin{itemize}[nosep]
\item For each class $\class \in \classes$, tree $\tree\in \forest$ and each  $\node\in \internalnodes[\tree]$, we have that $\nodesupport[\tree,\node,\class] = \nodesupport[\tree,\leftchild[\node],\class] + \nodesupport[\tree,\rightchild[\node],\class]$.
\item For each tree $\tree\in \forest$, we have $\sum\limits_{\node\in \leaves[\tree]} \sum\limits_{\class\in \classes}  \nodesupport[\tree,\node,\class] = \nexamples$.
\end{itemize}

For each tree $\tree \in \forest$, we let $\positivesplits[\tree,\node] \subseteq \{1..\nattributes\}$ denote the set of indices of the boolean attributes that must be \textsc{True} for an example to fall into node $\node \in \leaves[\tree]\cup\internalnodes[\tree]$. Similarly, $\negativesplits[\tree,\node]\subseteq \{1..\nattributes\}$ is the set of indices of the boolean attributes that must be \textsc{False} for an example to fall into  $\node$ (hence $\positivesplits[\tree,\node] \cap \negativesplits[\tree,\node] = \emptyset$). Both represent the splits that are found along the path from the root node of tree $\tree \in \forest$ to  $\node$.
Formally, if~$v$ is the root node of the tree, then 
$\positivesplits[\tree,\node]=\negativesplits[\tree,\node]=\emptyset$. 
For every $\node\in \internalnodes[\tree]$, we can define such sets for its children as
$\positivesplits[\tree,{\leftchild[\node]}]=\positivesplits[\tree,\node]$ and $\negativesplits[\tree,{\leftchild[\node]} ]=\negativesplits[\tree,\node]\cup\{\attrib[\node]\}$ ;
$\positivesplits[\tree,{\rightchild[\node]}]=\positivesplits[\tree,\node]\cup\{\attrib[\node]\}$ and $\negativesplits[\tree,{\rightchild[\node]} ]=\negativesplits[\tree,\node]$ 
 
The goal of (\probname) is to find 
$\nexamples$ vectors $x_1,..,x_{\nexamples}\in \{0,1\}^\nattributes$, respective classifications $z_1,..,z_{\nexamples}\in \classes$ and node incidences $y_{\tree\node\example}\in \Z_+, \forall \tree \in \forest$, $\node \in  \leaves[\tree]$, $\example\in \{1..\nexamples\}$ such that:

\begin{itemize}
\item $\sum\limits_{\node \in \leaves[\tree]} y_{\tree\node\example} \in \values$, $\forall \tree\in \forest, \example\in \{1..\nexamples\} $
\item \hspace*{-0.78cm} $\sum\limits_{k\in \{1,\ldots,\nexamples\}: z_{\example}=\class } y_{\tree\node\example} = \nodesupport[\tree,\node,\class]$,
$\forall \tree \in \forest, \node \in  \leaves[\tree]$
\item For all $k\in \{1..\nexamples\}$, $\tree \in \forest$, $\node \in  \leaves[\tree]$,  if $y_{\tree\node\example}>0$, then 
$(x_k)_i=0$ for all $i\in \negativesplits[\tree,\node]$  and
$(x_k)_i=1$ for all $i\in \positivesplits[\tree,\node]$.
\end{itemize}
We note that the last constraint implies that for every $\tree \in \forest$ and $\example \in \{1..\nexamples\}$, at most one variable in the set $\{y_{\tree\node\example}: \node \in  \leaves[\tree] \}$ is nonzero. This is due to the fact that, from the way $\positivesplits[\tree,\node]$ and $\negativesplits[\tree,\node]$ are constructed, all other leaves $\node'\in \leaves[\tree]\setminus\{\node\}$ must have at least one attribute $i\in \positivesplits[\tree,\node]\cup \negativesplits[\tree,\node]$ switching its required \textsc{True}/\textsc{False} value in $v'$.

Note that if we set 
$\values=\{1\}$,
we impose that each example must appear exactly once in every tree, which corresponds to the situation when no bagging is used. If we set  
$\values=\{0,\ldots,\nexamples\}$,
we get that each example can appear any number of times in a tree, which corresponds to the situation when bagging is used. 

\begin{theorem}
The decision version of (\probname) is $\np$-complete.
\label{thm:nphard}
\end{theorem}
\begin{proof}
First, a {\bf YES} certificate of (\probname) is any solution $(x,z,y)$. One can verify if it is feasible in polynomial time, so the decision version of $(\probname)\in \np$. 

Next, consider an instance of the $\np$-complete problem 3-SAT, given by a set  $\clauses=\{1,\ldots,|\clauses|\}$ of clauses with three literals each. Each literal is either a variable or its complement, and there are $\nsatvars$ possible variables. 
We construct an instance of (\probname) with $|\classes|=1$, by specifying the forest~$\forest$ and $\nodesupport[\clause,\node,\class]$ values. We let $\values=\{1\}$. By assumption on the data, we only need to specify $\nodesupport[\clause,\node,\class]$ for $\node \in \leaves[\tree]$. Also, recall that the left branch of $\node \in \internalnodes[\tree]$ corresponds to setting $\attrib[\node]$ to zero, and the right branch sets it to 1.
The idea is to construct a perfect binary tree for each $\clause \in \clauses$.
This way, the fixed attributes $\positivesplits[\clause,\node]$ and $\negativesplits[\clause,\node]$
of each leaf $\node$ of tree $\clause$ will represent one of each possible assignment of the three literals of $\clause$ to 0 or 1.

Out of the eight possible leaves of tree $\clause$, seven satisfy the corresponding clause, and one does not. We set $\nodesupport[\clause,\node,\class]=1$ for the leaves that satisfy clause $\clause$ and $\nodesupport[\clause,\node,\class]=0$ otherwise. The goal of our construction is to force the DRP solver to generate an example whose attribute values lead it towards one of the leaves with $\nodesupport[\clause,\node,\class]=1$ in each tree, i.e., a 3-SAT solution, whenever such a solution exists. To achieve this, we include six additional ``dummy'' examples for each tree to reach the remaining alternatives with $\nodesupport[\clause,\node,\class]=1$.
Therefore, our constructed instance has $\nexamples=6|\clauses| + 1$ examples.

We rely on $\nattributes = \nsatvars+|\clauses|$ binary attributes. The first~$\nsatvars$ attributes will represent the 0/1 assignment of values to literals. The remaining $|\clauses|$ attributes will determine whether one example is used or not in the leaves corresponding to the perfect binary tree of that clause. With that, we add to each tree $\clause=1,\ldots,|\clauses|$, a root node where we branch on feature $\nsatvars + \clause$. The right branch of the tree will contain the perfect binary tree described above. The left branch is a single leaf node with $\nodesupport[\clause,\node,\class]=6|\clauses| - 6$, designed to absorb all dummy examples not destined toward that tree.

We finally construct one extra auxiliary tree. All its left branches end up in leaves. Each node on the right branch at depth $\depth$ (including the root node) branches on attribute $\nsatvars + \depth + 1$.
The left leaf at depth 1 has 
$\nodesupport[\clause,\node,\class]=6|\clauses| - 6$. The left leaf at depth 2 has 
$\nodesupport[\clause,\node,\class]=6$. 
All other left leaves have 
$\nodesupport[\clause,\node,\class]=0$. 
The rightmost leaf has $\nodesupport[\clause,\node,\class]=1$. 
With this, the rightmost leaf imposes that a single example reaches one leaf of every perfect tree (the desired 3-SAT solution). The other nodes of the tree just count the extra examples.

Appendix~\ref{appendix:hardness} proves that (\probname) is feasible if and only if the 3-SAT instance is a {\bf YES} instance. It also contains illustrative examples for the construction.
\end{proof}

The optimization version of the problem is to search for the solution that has the largest likelihood,
called the \emph{maximum likelihood dataset reconstruction problem}
(\mlprobname). This problem is $\np$-hard since even reconstructing one feasible solution is $\np$-complete. The maximum likelihood objective function will be formally introduced in the next section.

\section{Constraint Programming Approach}
\label{sec:method_implementation}
\label{methodology}

As seen in Section~\ref{sec_illustrative_example}, an inspection of the different trees gives sets of restrictions over feature values that concern a known number of examples of each class. However, it does not tell which example specifically satisfies which condition. Testing feasible combinations by inspection would require extensive trial and error, leading to an intractable process. Instead, we propose to formulate this search problem as a constraint programming (CP) model, permitting the use of efficient out-of-the-shelf solvers for such models. The model we design covers the most general case where bagging is used to train the forest, and includes discretization strategies 
specifically designed to help the solution process. Note that, while we focus on CP, Mixed-Integer Linear Programming (MILP) could also be employed instead. However, having conducted experiments with both techniques, and as demonstrated in Appendix~\ref{appendix:milp_models}, CP generally achieved better performance and permitted to handle bagging much more effectively.

For our mathematical formulation, we define three sets of decision variables. The first one assigns training examples to a corresponding class. The second assigns the training examples to the trees' leaves, and the third connects the attributes' values to the splits leading to their assigned leaf.

\begin{itemize}
    \item $\forall \example \in \{1..\nexamples\}, \forall \class \in \classes$: $\varz[\example,\class]$ is 1 if training example $k$ is considered as part of class $\class$, else 0.
    \item $\forall \tree \in \forest, \forall \node \in \leaves[\tree], \forall \example \in \{1..\nexamples\}, \forall \class \in \classes$: $\varyb[\tree,\node,\example,\class] \in \Z_+$ is the number of times
     training example $\example$ is classified by leaf $\node$ within tree $\tree$ as class $\class$.
    \item $\forall \example \in \{1..\nexamples\}, \forall \feature \in \{1..\nattributes\}$: $\varx[\example,\feature] \in \{ 0;1 \}$ is the value of feature $\feature$ for example $\example$ in the reconstruction.
\end{itemize}
To define the objective function, we will assume that a training example appears at most $\maxval$ times in any tree, since higher values are very unlikely. 
Indeed, bootstrap sampling consists of sampling with replacement $\nexamples$ examples from a set of $\nexamples$ original examples. At each iteration of the bootstrap sampling process, each example has a probability of $\frac{1}{\nexamples}$ of being selected. The probability of an example being selected more than B times can then be computed as: $$1 - \sum\limits_{\val=0}^{B-1}\left(\left(\frac{1}{\nexamples}\right)^{\val} \cdot \left(\frac{\nexamples - 1}{\nexamples}\right)^{\nexamples - \val}\cdot \binom{\nexamples}{\val}\right).$$ If $\nexamples=100$ as in our experiments, the probability of an example appearing more than $7$ times in a bootstrap sampled training set is roughly $10^{-5}$. This value remains similar for larger values of $\nexamples$ (e.g., around $10^{-5}$ for $\nexamples=10^{10}$).

With this, we define $\values:=\{0,\ldots,\maxval\}$. Note that if no feasible solution is found using this default value, increasing it and solving the model again is possible. We now define a binary variable to capture how many times an example is used:
\begin{itemize}
    \item $\forall \tree \in \forest, \forall \example \in \{1..\nexamples\}, \forall \val \in \values$: $\varq[\tree,\example,\val]$ is 1 if training example $\example$ is used $\val$ times in tree $\tree$ and 0 otherwise
\end{itemize}

The constraints of our model are as follows. 

One-hot encoding:
\begin{itemize}
    \item $\forall \example \in \{1..\nexamples\}, \forall \ohegroup \in \ohevects: \quad \sum\limits_{\feature \in \ohegroup} \varx[\example,\feature] = 1$
\end{itemize}
Each example is assigned to exactly one class:
\begin{itemize}
    \item $\forall \example \in \{1..\nexamples\}: \quad \sum\limits_{\class \in  \classes} \varz[\example,\class] = 1$ %
\end{itemize}
If an example is not assigned a given class, it cannot be used as that class in any tree:
\begin{itemize}
    \item $\forall \example \in \{1..\nexamples\}, \forall \class \in  \classes:$
     
     $\textbf{if } \varz[\example,\class]=0 \textbf{ then} \sum\limits_{\tree \in \forest, \node \in  \leaves[\tree]} \varyb[\tree,\node,\example,\class] = 0$ %
\end{itemize}
Each leaf must capture exactly the defined number of examples from each class:
\begin{itemize}
 \item $\forall \tree \in \forest, \forall \node \in  \leaves[\tree], \forall \class \in \classes: \nodesupport[\tree,\node,\class] = \sum\limits_{{\example \in \{1..\nexamples\}}} \varyb[\tree,\node,\example,\class]$

\end{itemize}
If an example is captured by a leaf, the associated conditions must be enforced on its features:
\begin{itemize}
 \item $\forall \tree \in \forest, \forall \example \in \{1..\nexamples\}, \forall \node \in  \leaves[\tree]:$

 $\textbf{if } \sum\limits_{\class \in \classes}\varyb[\tree,\node,\example,\class] \geq 1 \textbf{ then} \quad \left(\bigwedge\limits_{\feature \in \positivesplits[\tree,\node]}{\varx[\example,\feature] = 1}\right) \land \left(\bigwedge\limits_{\feature \in \negativesplits[\tree,\node]}{\varx[\example,\feature] = 0}\right)$
\end{itemize}

The number of times a sample is used in a tree is consistent:
\begin{itemize}
\item $\forall \tree \in \forest, \forall \val \in \values, \forall \example\in \{1..\nexamples\}:$

$\sum\limits_{\node \in \leaves[\tree],\class \in \classes} \varyb[\tree,\node,\example,\class] = \val \iff \varq[\tree,\example,\val]=1$

\end{itemize}

We implemented this model using the \texttt{OR-Tools} CP-SAT solver~\citep{cpsatlp}, which requires extra variables and constraints to be introduced to model some of the above conditions. These details are presented in Appendix~\ref{appendix:cpdetails}.

\paragraph{Maximum log-likelihood objective.}
Since the above model could have many possible solutions when using bagging, we orient the search towards the solutions (datasets) that are the most likely. 
For a given tree $\tree$, let $p_{\val}$ be the probability that a sample $\example$ is chosen exactly $b$ times to train that tree.
By defining 
$p^q_{\tree\example\val}=p_{\val}$ if $\varq[\tree,\example,\val]=1$ and $p^q_{\tree\example\val}=1$ otherwise,
we can calculate the probability that the samples were chosen for the tree according to the $\varq[\tree,\example,\val]$ variables as
$\prod_{\example \in \{1..\nexamples\}} \prod_{\val \in \values} p^q_{\tree\example\val}.$ Therefore, considering the whole RF, the probability of a given solution is:
$$\prod\limits_{\tree \in \forest}\prod\limits_{\example \in \{1..\nexamples\}}  \prod\limits_{\val \in \values} p^q_{\tree\example\val}.$$
Maximizing this probability is equivalent to maximizing its logarithm; in other words, maximizing:
$$\sum\limits_{\tree \in \forest}\sum\limits_{\example \in \{1..\nexamples\}} \sum\limits_{\val \in \values} \log\left( p^q_{\tree\example\val} \right)=
 \sum\limits_{\tree \in \forest}\sum\limits_{\example \in \{1..\nexamples\}} \sum\limits_{\val \in \values} \log\left( p_{\val} \right) \varq[\tree,\example,\val].$$

\paragraph{Model simplifications when bagging is deactivated.} 
RFs can be trained using random subsets of features for each split but considering all the examples in each tree. In such situations without bagging, the CP model can be significantly simplified. Variables $\varyb[\tree,\node,\example,\class]$ will become binary and 
sum up to $1$ for each tree $\tree \in \forest$ and each example $\example \in \{1..\nexamples\}$, since each example will be used exactly once in each tree. 
Also, we know in advance how many examples are of each class $\class\in \classes$. Therefore, to match that data, we may fix the variables $\varz[\example,\class]$ in advance.
Finally, $\varq[\tree,\example,\val]$ are always fixed since every example is used exactly once in each tree. Thus, the objective function becomes constant, and the problem reduces to the search for a feasible solution.

\paragraph{Reconstructing non-binary attributes.} To streamline the presentation of the methodology and evaluation metric, we provided our model for the particular case of binary attributes. However, extending it to handle other types of attributes is possible with minimal changes, as detailed in Appendix~\ref{appendix:other_types_of_attributes} and implemented within our publicly available repository. In a nutshell, \emph{categorical} attributes are typically 
one-hot encoded for tree ensembles and directly handled by our formulation. \emph{Ordinal} features can be modeled as integer variables and only require a slight generalization of the constraints connecting the attributes' values to the assignment of the examples to the leaves. Finally, \emph{numerical} attributes can also be reconstructed: although they take values in a continuous space, the number of splits within the forest is finite, and so is the number of different intervals in which they can lie. Leveraging this observation, we can use ordinal features to model such possible intervals in the reconstruction. 

\section{Experimental Study}
\label{sec:experiments}

Through extensive experimental analyses, we aim to evaluate the effectiveness and accuracy of the proposed reconstruction attack, named \draft{} \emph{(Dataset Reconstruction Attack From Trained ensembles)}.
We first detail the experimental setup before discussing the results.

\subsection{Experimental Setup}
\label{subsec:expes_setup}

\paragraph{Datasets.} We rely on three popular datasets for binary classification in our experiments. We discretize each dataset's numerical attributes and one-hot encode the categorical ones. To keep a reasonably small number of features, we remove some attributes with the smallest support\footnote{Our binarized versions of these datasets are available in the supplementary material and will be available on our online repository upon publication.}.
First, the COMPAS dataset (analyzed by~\citealt{angwin2016machine}) gathers records about criminal offenders in the Broward County of Florida collected from 2013 and 2014, with the task being recidivism prediction. Our preprocessed version includes $7{,}206$ examples described by $15$ binary attributes.
Second, the UCI Adult Income dataset~\citep{Dua:2019} contains data regarding the 1994 U.S. census to predict whether a person earns more than \$50K/year.
After preprocessing, our dataset includes $48{,}842$ examples and $20$ binary features. 
Finally, we use the Default of Credit Card Client dataset~\citep{yeh2009comparisons}, to predict whether a person will default in payment (the next time they use their credit card). Our preprocessed version includes $29{,}986$ examples and $22$ binary attributes.

\paragraph{Reconstruction error evaluation.} To assess the attack's success, 
we first compute the Manhattan distance between each reconstructed and original example. 
The resulting distance matrix then instantiates a minimum weight matching in bipartite graphs, also known as linear sum assignment problem, which we solve using the \scipy~\cite{2020SciPy-NMeth} Python library. Once the datasets are aligned, we then measure the proportion of binary attributes that differ between both.

\paragraph{Random reconstruction baseline.} As mentioned in Section~\ref{sec:rel_works}, reconstruction attacks rarely target reconstructing an entire training set, and none of them apply to our setup (\emph{i.e.,} leveraging an RF to rebuild its complete training set). We then consider a baseline adversary with the same knowledge as ours (in particular, the number of examples $\nexamples$, the different attributes $\nattributes$ including their one-hot encoding $\ohevects$) except for the RF itself. 
The adversary then randomly guesses each attribute of each example, remaining consistent with the one-hot encoding information.
The reconstruction error is finally assessed, as described in the previous paragraph.
We average such computation over $100$ random runs and report the average value. By comparing this baseline with the performances of our approach, one can then quantify how much additional information can be extracted from~the~RF.

\paragraph{Target RFs.} To train our target models (\emph{i.e.,} the RFs from which we attempt to reconstruct the training data), we use the popular implementation provided by the \sklearn{} library. 
For each dataset, we learn RFs with varying parameters. More precisely, we use a number of trees $\lvert \forest \rvert \in \{1, 5, 10, 20, 30, 40, 50, 60, 70, 80, 90, 100\}$ with maximum depth $\depth_{max} \in \{\text{None},2,3,4,5,10\}$ (where $\text{None}$ stands for no maximum depth constraint).
For each experiment, we randomly sample $100$ examples from the entire dataset to form a training set, and use the remaining ones as a test set to verify to what extent the models generalize. 
We repeat the experiment five times using different seeds for the random sampling, and report the average results and their standard deviation across the five runs. 

\begin{figure*}[h!]
  \centering
  \begin{subfigure}[t]{0.48\textwidth}
    \centering\includegraphics[width=0.93\textwidth]{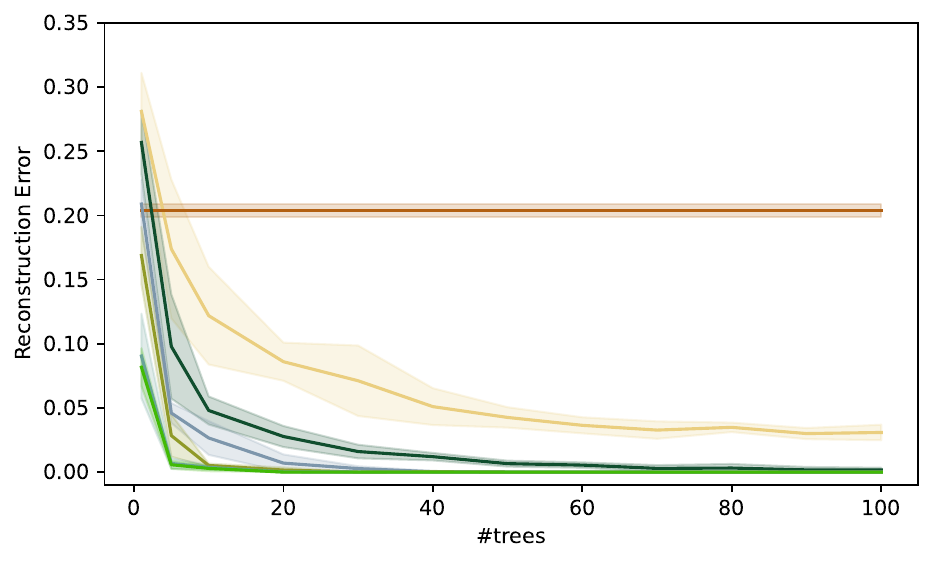}
    \caption{COMPAS dataset, bagging not used}\label{fig:results_compas_no_bagging}
  \end{subfigure}
  \begin{subfigure}[t]{0.48\textwidth}
    \centering\includegraphics[width=0.93\textwidth]{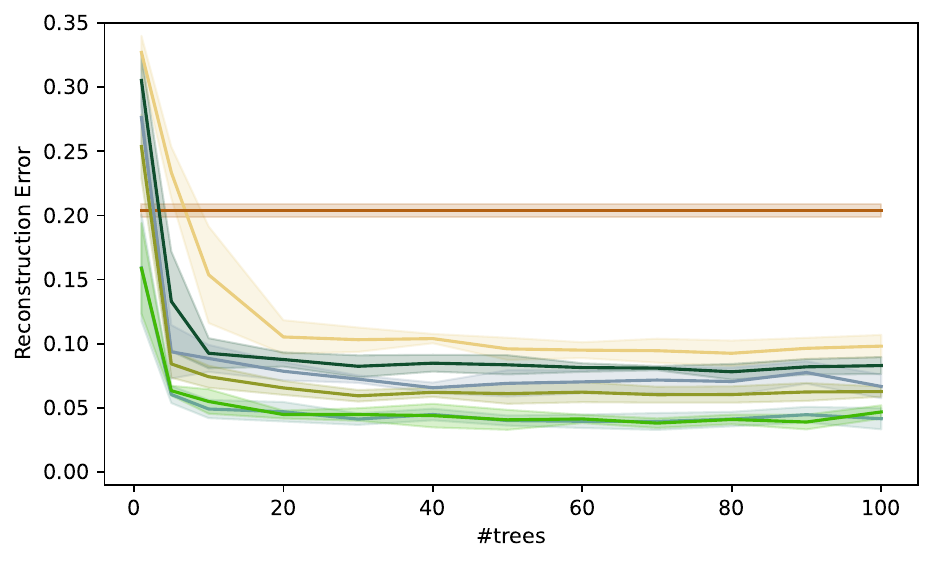}
    \caption{COMPAS dataset, bagging used}\label{fig:results_compas_bagging}
  \end{subfigure}
    \begin{subfigure}[t]{0.48\textwidth}
    \centering\includegraphics[width=0.93\textwidth]{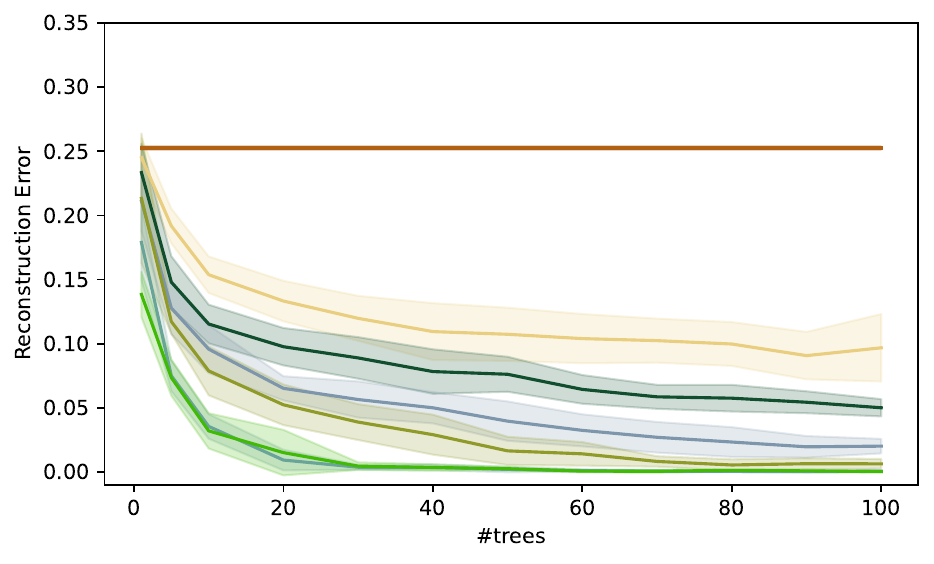}
    \caption{UCI Adult Income dataset, bagging not used}\label{fig:results_adult_no_bagging}
  \end{subfigure}
  \begin{subfigure}[t]{0.48\textwidth}
    \centering\includegraphics[width=0.93\textwidth]{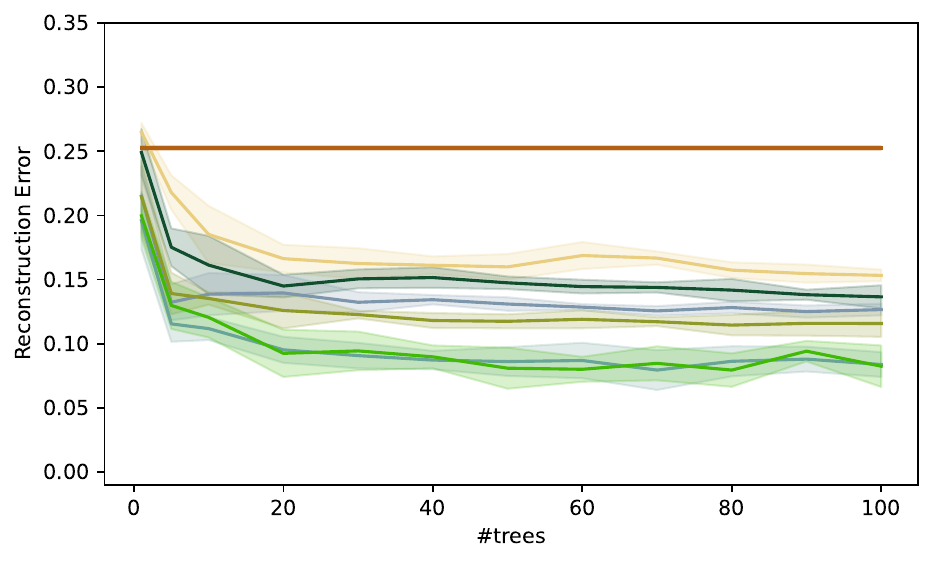}
    \caption{UCI Adult Income dataset, bagging used}\label{fig:results_adult_bagging}
  \end{subfigure}
      \begin{subfigure}[t]{0.48\textwidth}
    \centering\includegraphics[width=0.93\textwidth]{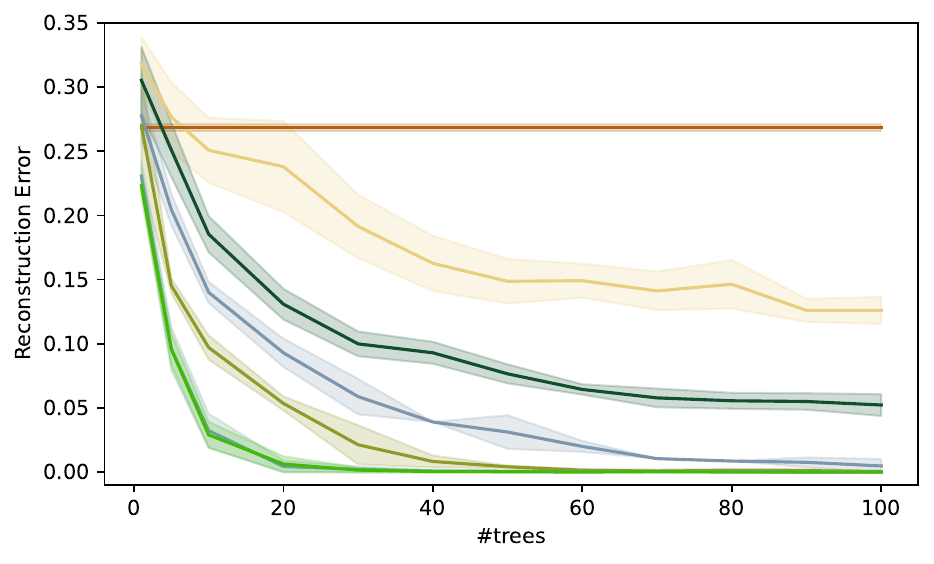}
    \caption{Default of Credit Card Client dataset, bagging not used}\label{fig:results_default_credit_no_bagging}
  \end{subfigure}
  \begin{subfigure}[t]{0.48\textwidth}
    \centering\includegraphics[width=0.93\textwidth]{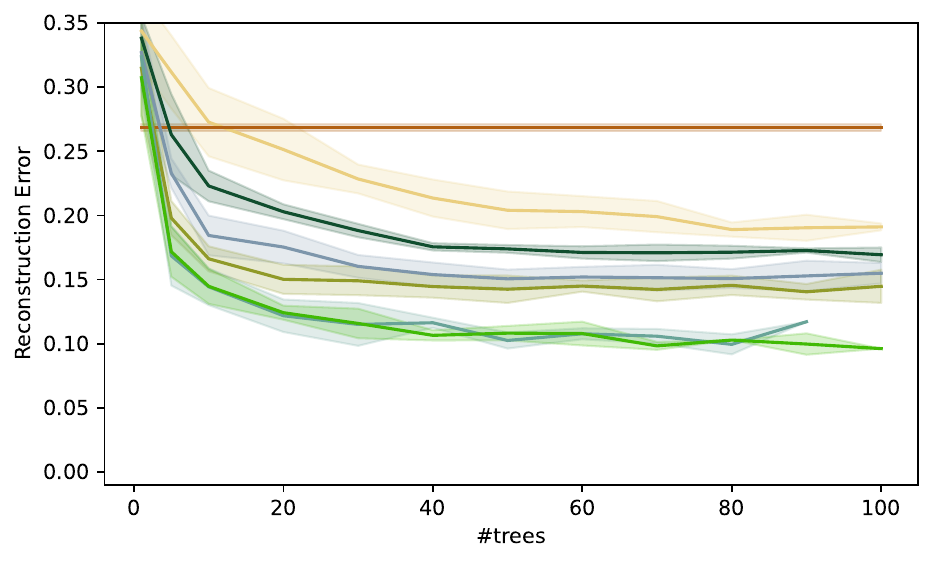}
    \caption{Default of Credit Card Client dataset, bagging used}\label{fig:results_default_credit_bagging}
  \end{subfigure}
  \includegraphics[width=0.75\textwidth]{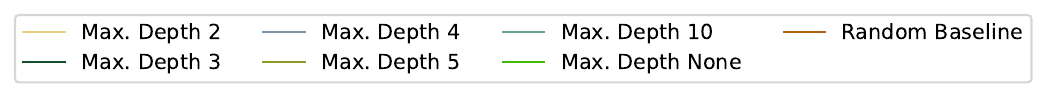}
    \caption{Average reconstruction error as a function of the number of trees $\lvert \forest \rvert$ within the target forest $\forest$, for different maximum depth values $\depth_{max}$ and for the random baseline.}
    \label{fig:results_all}
\end{figure*}

\paragraph{Dataset reconstruction.} The proposed CP models described in Section~\ref{sec:method_implementation} are solved using the \ortools{} CP-SAT solver~\citep{cpsatlp} (v9).
Each model resolution is limited to a maximum of five hours of CPU time using $16$ threads with up to $6$ GB of RAM for each thread. Note, however, that while the CP models handling bagging often reached this time limit, they usually were able to find feasible solutions in a much shorter time. 
All experiments are run on a computing cluster over a set of homogeneous nodes using Intel Platinum 8260 Cascade Lake @ 2.4GHz CPU.

All the material (source code and data sets) needed to reproduce our experiments is accessible at \url{https://github.com/vidalt/DRAFT} under a MIT license. 

\subsection{Results}
\label{subsec:results}

The results of our experiments are reported in Figure~\ref{fig:results_all} 
for all three datasets, with or without the use of bagging to train the target RFs.
More precisely, we plot the average reconstruction error as a function of the number of trees $\lvert \forest \rvert$, for several values of the trees' maximum depth~$\depth_{max}$. We observe several trends that are consistent across all three datasets. In all cases, as expected, increasing the trees' depth or the number of trees in the forest decreases the reconstruction error as it provides more information regarding the training data. When bagging is not used to train the RFs, the reconstruction error reaches $0$ in all cases for the deepest forests (recall that the default parameters of the \sklearn{}
library is no maximum depth constraint). This is not the case when using bagging. In such cases, the reconstruction error reaches a threshold and stops improving even for larger forests.
In Appendix~\ref{appendix:benchmark}, we further investigate the effect of bagging on protecting the training data against reconstruction attacks and conduct additional experiments. Our main finding is that this performance drop precisely comes from the difficulty of guessing how many times each example went through each tree: bagging intrinsically provides a form of protection regarding the training data. This is consistent with theoretical results stating that bagging provides (weak) differential privacy guarantees~\citep{DBLP:conf/ijcai/LiuJG21}.

\begin{table}[tb]
\centering
\caption{Average run time and number of runs for which the solver did not come up with a feasible solution (\textbf{\#Runs}), for a fixed number of trees $\lvert \forest \rvert = 100$ (default value).}\label{tab:runtimes_timeouts_fixed_n_trees}
\centering
\scalebox{0.95}
{
\begin{tabular}{c|ccccc}
\cline{2-5}
&\multicolumn{1}{c|}{\multirow{2}{*}{\textbf{\begin{tabular}[c]{@{}c@{}}Max.\\ Depth\end{tabular}}}} & \multicolumn{2}{c|}{\textbf{No Bagging}}                           & \textbf{Bagging}    \\ \cline{3-5} 
&\multicolumn{1}{c|}{}                                                                               & \textbf{Avg. T (s)} & \multicolumn{1}{c|}{\textbf{\#Runs}} & \textbf{\#Runs} \\ \hline
\parbox[t]{2mm}{\multirow{6}{*}{\rotatebox[origin=c]{90}{\textbf{COMPAS}}}} 
&2                                                                                                   & 9.5                     & 0                                        & 0                   \\
&3                                                                                                   & 36.5                    & 0                                        & 0                   \\
&4                                                                                                   & 45.7                    & 0                                        & 0                   \\
&5                                                                                                   & 70.0                    & 0                                        & 0                   \\
&10                                                                                                  & 110.9                    & 0                                        & 0                   \\
&None                                                                                                & 100.8                    & 0                                        & 0                   \\ \hline
\parbox[t]{2mm}{\multirow{6}{*}{\rotatebox[origin=c]{90}{\textbf{Adult}}}}
&2                                                                                                   & 20.2                     & 0                                        & 0                   \\
&3                                                                                                   & 81.5                   & 0                                        & 0                   \\
&4                                                                                                   & 1943.5                   & 0                                     & 0                   \\
&5                                                                                                   & 1290.7                   & 0                                     & 0                   \\
&10                                                                                                  & 346.4                   & 0                                        & 1/5                \\
&None                                                                                                & 196.3                   & 0                                        & 1/5                \\ \hline
\parbox[t]{2mm}{\multirow{6}{*}{\rotatebox[origin=c]{90}{\textbf{Default Credit}}}}
&2                                                                                                   & 35.5                    & 0                                        & 0                   \\
&3                                                                                                   & 300.8                   & 1/5                                     & 0                   \\
&4                                                                                                   & 6040.0                  & 2/5                                    & 0                   \\
&5                                                                                                   & 1358.5                  & 0                                     & 0                   \\
&10                                                                                                  & 382.8                   & 0                                        & 0               \\
&None                                                                                                & 165.0                    & 0                                        & 4/5               \\ \hline
\end{tabular}
}
\end{table}
\begin{table}[tb]
\centering
\caption{Average run time and number of runs for which the solver did not come up with a feasible solution (\textbf{\#Runs}), for no maximum depth constraint (default value).}\label{tab:runtimes_timeouts_fixed_max_depth}
\centering
\scalebox{0.95}
{
\begin{tabular}{c|ccccc}
\cline{2-5}
&\multicolumn{1}{c|}{\multirow{2}{*}{\textbf{\begin{tabular}[c]{@{}c@{}}$\lvert \forest \rvert$\end{tabular}}}} & \multicolumn{2}{c|}{\textbf{No Bagging}}                           & \textbf{Bagging}    \\ \cline{3-5} 
&\multicolumn{1}{c|}{}                                                                               & \textbf{Avg. T (s)} & \multicolumn{1}{c|}{\textbf{\#Runs}} & \textbf{\#Runs} \\ \hline
\parbox[t]{2mm}{\multirow{6}{*}{\rotatebox[origin=c]{90}{\textbf{COMPAS}}}} 
&1                                                                                                   & 0.7                     & 0                                        & 0                   \\
&10                                                                                                   & 8.4                    & 0                                        & 0                   \\
&30                                                                                                   & 39.8                    & 0                                        & 0                   \\
&50                                                                                                   & 53.0                    & 0                                        & 0                   \\
&80                                                                                                  & 89.7                    & 0                                        & 0                   \\
&100                                                                                                & 100.8                    & 0                                        & 0                   \\ \hline
\parbox[t]{2mm}{\multirow{6}{*}{\rotatebox[origin=c]{90}{\textbf{Adult}}}}
&1                                                                                                   & 0.8                     & 0                                        & 0                   \\
&10                                                                                                   & 74.2                   & 0                                        & 0                   \\
&30                                                                                                    & 84.9                   & 0                                     & 0                   \\
&50                                                                                                   & 85.1                   & 0                                     & 0                   \\
&80                                                                                                  & 119.9                   & 0                                        & 0                \\
&100                                                                                                & 196.3                   & 0                                        & 1/5                \\ \hline
\parbox[t]{2mm}{\multirow{6}{*}{\rotatebox[origin=c]{90}{\textbf{Default Credit}}}}
&1                                                                                                   & 0.9                    & 0                                        & 0                   \\
&10                                                                                                   & 129.7                   & 0                                     & 0                   \\
&30                                                                                                   & 47.7                  & 0                                    & 0                   \\
&50                                                                                                   & 66.2                  & 0                                     & 2/5                   \\
&80                                                                                                  & 161.3                   & 0                                        & 4/5               \\
&100                                                                                                & 165.0                    & 0                                        & 4/5               \\ \hline
\end{tabular}
}
\end{table}

We report in Tables~\ref{tab:runtimes_timeouts_fixed_n_trees} and~\ref{tab:runtimes_timeouts_fixed_max_depth} (respectively for a fixed number of trees $\lvert \forest \rvert = 100$ and no fixed maximum depth, both corresponding to \sklearn{}'s default values) the average run times for the reconstruction model without bagging, along with the number of times (\#Runs) the solver was unable to find a feasible solution before timeout. When bagging is used, most runs attain the time limit and return a solution but cannot prove optimality. Run times in this context are not informative, so we only report the number of times the solver did not find any feasible solution before the time limit. Note that the few runs that did not produce a feasible solution are excluded from Figure~\ref{fig:results_all}.
We observe from Table~\ref{tab:runtimes_timeouts_fixed_n_trees} that the formulation without bagging efficiently handles the problems that are under-constrained (shallow trees) or over-constrained (deep trees). Intermediate cases seem to require more computational effort, and in a few cases, the solver did not find a feasible solution.
When using bagging, the size of the models seems to matter the most, as the solver only failed to find feasible solutions with the deepest forests.
The same observation holds from Table~\ref{tab:runtimes_timeouts_fixed_max_depth}, as the only runs for which the solver did not find a feasible solution are those with the largest numbers of trees. When not using bagging, the solution times scale approximately linearly with the number of trees. We report in Appendix~\ref{appendix:scalability_reconstr} additional experiments regarding our method's scalability with respect to the number of training examples $\nexamples$. The results demonstrate its ability to reconstruct considerably larger datasets, with the reconstruction error remaining very small. Furthermore, while the size of the CP model's search space increases exponentially with the number of reconstructed training examples $\nexamples$, in practice, reconstruction time increases polynomially (approximately quadratic or sub-quadratic) with $\nexamples$.

As discussed in Section~\ref{sec:rel_works}, most works in the reconstruction attacks literature only target reconstructing part of the dataset attributes (generally, a single one), assuming the others are publicly known. In Appendix~\ref{appendix:partial_reconstr}, we perform complementary experiments on such partial reconstruction. The results show that our approach successfully leverages knowledge of part of the dataset attributes, which results in lower error rates for the other ones.

\section{Discussion and Conclusions}
\label{sec:conclusion}

This study has shown that the structure of a trained RF can be exploited to reconstruct most (if not all) of its training data. It introduced a new paradigm of attack, leveraging mathematical programming tools to encode the structure of an RF and relying on a general-purpose CP solver to find the most likely reconstructions of the training data.
Due to the high redundancy of RFs built using off-the-shelf ML libraries with their default parameters, the resulting problem is often strongly constrained, resulting in a high reconstruction rate. While theoretical $\np$-completeness theorems indicate that such an attack may not be computationally tractable at scale, the tremendous progress in CP/MILP solvers has made it practical to solve larger and larger problems over time. Therefore, it may just be a question of time until data breaches happen for large datasets.

The fact that the proposed framework is based on mathematical programming techniques opens the door to many promising research perspectives. The approach could be tested on various types of attributes (numerical, categorical) without the need for feature binarization.
Performance improvements could also be achieved through different problem reformulations or additional valid inequalities. Notably, one could leverage the information gain criterion used to select the splits while building the decision tree to eliminate combinations of attributes' values leading to different splits.

Our framework can also be used as a building block for other types of inference attacks, such as membership inference or property inference. Furthermore, we considered canonical RFs trained without privacy-preserving techniques, representing most of what popular libraries do by default. Investigating the effectiveness of common privacy-preserving mechanisms, such as the widely used differential privacy~\citep{dwork2014algorithmic}, would bring additional insights.
Though this may lead to difficult models, the proposed CP (or MILP) formulations could be extended to infer the noise added by the protection mechanisms on the released per-node counts~\citep{dp-trees-survey,DBLP:conf/pods/DinurN03}. On the same line, adapting the formulation to work without the knowledge of the per-leaf per-class counts (hence only supposing that each leaf contains at least one example from the predicted class), or considering other gray-box setups, are interesting directions. Finally, another interesting direction is to apply the proposed methodology to other types of ensembles 
and ML models.

\section*{Impact Statement}

ML models are commonly trained using large amounts of data, often including personal or private information. The flourishing literature on inference attacks against ML models showed that models might jeopardize their training data even when accessed in a black-box manner (\emph{i.e.,} through a prediction API). Furthermore, transparency requirements encourage practitioners to either provide additional explanations for their model's decisions or to entirely release such models, potentially opening up to new attacks.

In this study, we have demonstrated that the structure of a trained RF can be leveraged to reconstruct most (if not all) of its training data. Importantly, our proposed method only leverages the information provided by popular libraries such as \sklearn{}. While NP-harness theorems and scalability issues limit the current applicability of our approach, our results already demonstrate its effectiveness on datasets of practical significance. These findings underscore a critical vulnerability inherent to widely adopted ensemble methods, warranting attention and mitigation. The methods and experiments developed in this study have two main implications: (i) raising awareness against the privacy vulnerabilities of ensemble methods and (ii) providing promising research paths to stress test privacy-preserving mechanisms, aiming to protect such models before releasing them.



\newpage
\appendix
\onecolumn

\section{Proof of NP-Hardness (Theorem~\ref{thm:nphard})}\label{appendix:hardness}

The description of the construction of the instance of (\probname) can be found in the main text. We start by providing an example to clarify the construction. 

Consider the following 3-SAT instance with $|\clauses| = 3$ clauses and $\nsatvars = 4$ variables:
\begin{equation}
( u_1 \vee \bar{u_2} \vee \bar{u_3} ) \wedge ( u_1 \vee u_2 \vee u_4  ) \wedge ( \bar{u_2} \vee \bar{u_3} \vee \bar{u_4} )
\label{eq:3satex}
\end{equation}

Figure~\ref{fig:3satex} shows the constructed (\probname) instance arising from it. 

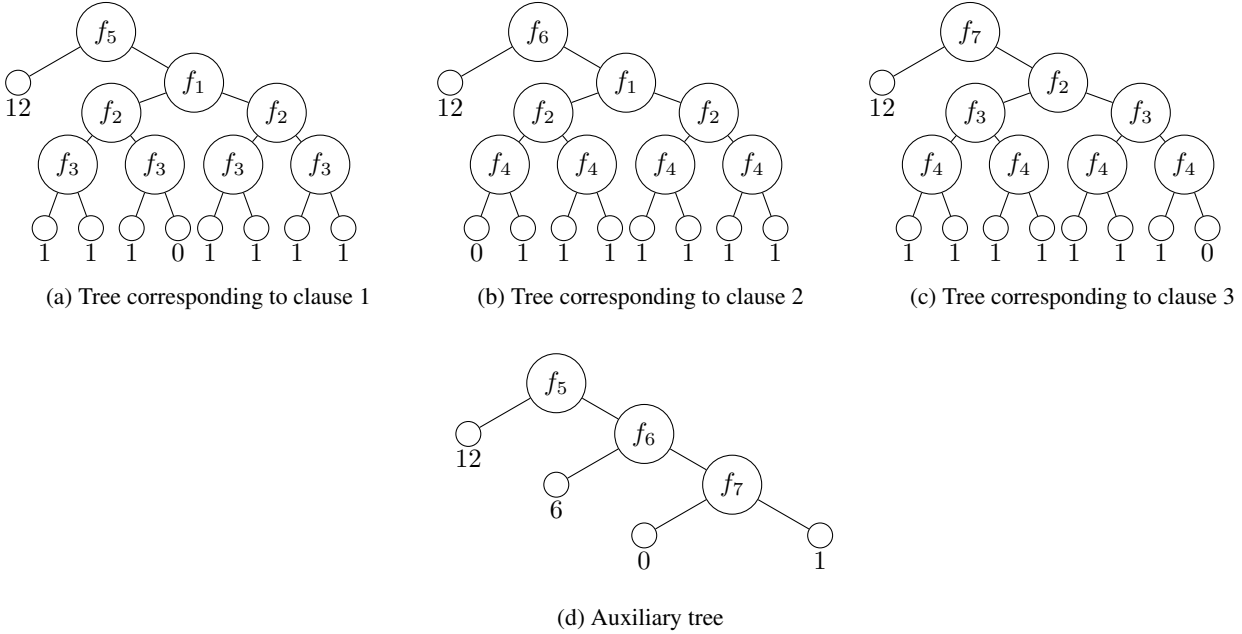
\begin{figure}[h!]
\begin{subfigure}[t]{0.33\textwidth}
\begin{tikzpicture}[scale=0.6,
    node/.style={%
      draw,
      circle,
    },
  ]

    \node [node] (v0) {$f_5$};
    \path (v0) ++(-150:1.5\nodeDist) node [node] (v1) {};
    \path (v0) ++(-30:1.5\nodeDist) node [node] (v2) {$f_1$};
    \path (v2) ++(-160:1.3\nodeDist) node [node] (v3) {$f_2$};
    \path (v2) ++(-20:1.3\nodeDist) node [node] (v4) {$f_2$};
    \path (v3) ++(-130:\nodeDist) node [node] (v5) {$f_3$};
    \path (v3) ++(-50:\nodeDist) node [node] (v6) {$f_3$};
    \path (v4) ++(-130:\nodeDist) node [node] (v7) {$f_3$};
    \path (v4) ++(-50:\nodeDist) node [node] (v8) {$f_3$};
    \path (v5) ++(-110:\nodeDist) node [node] (v9) {};
    \path (v5) ++(-70:\nodeDist) node [node] (v10) {};
    \path (v6) ++(-110:\nodeDist) node [node] (v11) {};
    \path (v6) ++(-70:\nodeDist) node [node] (v12) {};
    \path (v7) ++(-110:\nodeDist) node [node] (v13) {};
    \path (v7) ++(-70:\nodeDist) node [node] (v14) {};
    \path (v8) ++(-110:\nodeDist) node [node] (v15) {};
    \path (v8) ++(-70:\nodeDist) node [node] (v16) {};  
    \path (v1) ++(270:0.1\nodeDist) node [below] {$12$};
    \path (v9) ++(270:0.1\nodeDist) node [below] {$1$};
    \path (v10) ++(270:0.1\nodeDist) node [below] {$1$};
    \path (v11) ++(270:0.1\nodeDist) node [below] {$1$};
    \path (v12) ++(270:0.1\nodeDist) node [below] {$0$};
    \path (v13) ++(270:0.1\nodeDist) node [below] {$1$};
    \path (v14) ++(270:0.1\nodeDist) node [below] {$1$};
    \path (v15) ++(270:0.1\nodeDist) node [below] {$1$};
    \path (v16) ++(270:0.1\nodeDist) node [below] {$1$};
    
    \draw (v0) -- (v1) ;
    \draw (v0) -- (v2) ;
    \draw (v2) -- (v3) ;
    \draw (v2) -- (v4) ;
    \draw (v3) -- (v5) ;
    \draw (v3) -- (v6) ;
    \draw (v4) -- (v7) ;
    \draw (v4) -- (v8) ;
    \draw (v5) -- (v9) ;
    \draw (v5) -- (v10) ;
    \draw (v6) -- (v11) ;
    \draw (v6) -- (v12) ;
    \draw (v7) -- (v13) ;
    \draw (v7) -- (v14) ;
    \draw (v8) -- (v15) ;
    \draw (v8) -- (v16) ;

\end{tikzpicture}
\caption{Tree corresponding to clause 1}
\end{subfigure}
\begin{subfigure}[t]{0.33\textwidth}
\begin{tikzpicture}[scale=0.6,
    node/.style={%
      draw,
      circle,
    },
  ]

    \node [node] (v0) {$f_6$};
    \path (v0) ++(-150:1.5\nodeDist) node [node] (v1) {};
    \path (v0) ++(-30:1.5\nodeDist) node [node] (v2) {$f_1$};
    \path (v2) ++(-160:1.3\nodeDist) node [node] (v3) {$f_2$};
    \path (v2) ++(-20:1.3\nodeDist) node [node] (v4) {$f_2$};
    \path (v3) ++(-130:\nodeDist) node [node] (v5) {$f_4$};
    \path (v3) ++(-50:\nodeDist) node [node] (v6) {$f_4$};
    \path (v4) ++(-130:\nodeDist) node [node] (v7) {$f_4$};
    \path (v4) ++(-50:\nodeDist) node [node] (v8) {$f_4$};
    \path (v5) ++(-110:\nodeDist) node [node] (v9) {};
    \path (v5) ++(-70:\nodeDist) node [node] (v10) {};
    \path (v6) ++(-110:\nodeDist) node [node] (v11) {};
    \path (v6) ++(-70:\nodeDist) node [node] (v12) {};
    \path (v7) ++(-110:\nodeDist) node [node] (v13) {};
    \path (v7) ++(-70:\nodeDist) node [node] (v14) {};
    \path (v8) ++(-110:\nodeDist) node [node] (v15) {};
    \path (v8) ++(-70:\nodeDist) node [node] (v16) {};  
    \path (v1) ++(270:0.1\nodeDist) node [below] {$12$};
    \path (v9) ++(270:0.1\nodeDist) node [below] {$0$};
    \path (v10) ++(270:0.1\nodeDist) node [below] {$1$};
    \path (v11) ++(270:0.1\nodeDist) node [below] {$1$};
    \path (v12) ++(270:0.1\nodeDist) node [below] {$1$};
    \path (v13) ++(270:0.1\nodeDist) node [below] {$1$};
    \path (v14) ++(270:0.1\nodeDist) node [below] {$1$};
    \path (v15) ++(270:0.1\nodeDist) node [below] {$1$};
    \path (v16) ++(270:0.1\nodeDist) node [below] {$1$};
    
    \draw (v0) -- (v1) ;
    \draw (v0) -- (v2) ;
    \draw (v2) -- (v3) ;
    \draw (v2) -- (v4) ;
    \draw (v3) -- (v5) ;
    \draw (v3) -- (v6) ;
    \draw (v4) -- (v7) ;
    \draw (v4) -- (v8) ;
    \draw (v5) -- (v9) ;
    \draw (v5) -- (v10) ;
    \draw (v6) -- (v11) ;
    \draw (v6) -- (v12) ;
    \draw (v7) -- (v13) ;
    \draw (v7) -- (v14) ;
    \draw (v8) -- (v15) ;
    \draw (v8) -- (v16) ;

\end{tikzpicture}
\caption{Tree corresponding to clause 2}
\end{subfigure}
\begin{subfigure}[t]{0.33\textwidth}
\begin{tikzpicture}[scale=0.6,
    node/.style={%
      draw,
      circle,
    },
  ]

    \node [node] (v0) {$f_7$};
    \path (v0) ++(-150:1.5\nodeDist) node [node] (v1) {};
    \path (v0) ++(-30:1.5\nodeDist) node [node] (v2) {$f_2$};
    \path (v2) ++(-160:1.3\nodeDist) node [node] (v3) {$f_3$};
    \path (v2) ++(-20:1.3\nodeDist) node [node] (v4) {$f_3$};
    \path (v3) ++(-130:\nodeDist) node [node] (v5) {$f_4$};
    \path (v3) ++(-50:\nodeDist) node [node] (v6) {$f_4$};
    \path (v4) ++(-130:\nodeDist) node [node] (v7) {$f_4$};
    \path (v4) ++(-50:\nodeDist) node [node] (v8) {$f_4$};
    \path (v5) ++(-110:\nodeDist) node [node] (v9) {};
    \path (v5) ++(-70:\nodeDist) node [node] (v10) {};
    \path (v6) ++(-110:\nodeDist) node [node] (v11) {};
    \path (v6) ++(-70:\nodeDist) node [node] (v12) {};
    \path (v7) ++(-110:\nodeDist) node [node] (v13) {};
    \path (v7) ++(-70:\nodeDist) node [node] (v14) {};
    \path (v8) ++(-110:\nodeDist) node [node] (v15) {};
    \path (v8) ++(-70:\nodeDist) node [node] (v16) {};  
    \path (v1) ++(270:0.1\nodeDist) node [below] {$12$};
    \path (v9) ++(270:0.1\nodeDist) node [below] {$1$};
    \path (v10) ++(270:0.1\nodeDist) node [below] {$1$};
    \path (v11) ++(270:0.1\nodeDist) node [below] {$1$};
    \path (v12) ++(270:0.1\nodeDist) node [below] {$1$};
    \path (v13) ++(270:0.1\nodeDist) node [below] {$1$};
    \path (v14) ++(270:0.1\nodeDist) node [below] {$1$};
    \path (v15) ++(270:0.1\nodeDist) node [below] {$1$};
    \path (v16) ++(270:0.1\nodeDist) node [below] {$0$};
    
    \draw (v0) -- (v1) ;
    \draw (v0) -- (v2) ;
    \draw (v2) -- (v3) ;
    \draw (v2) -- (v4) ;
    \draw (v3) -- (v5) ;
    \draw (v3) -- (v6) ;
    \draw (v4) -- (v7) ;
    \draw (v4) -- (v8) ;
    \draw (v5) -- (v9) ;
    \draw (v5) -- (v10) ;
    \draw (v6) -- (v11) ;
    \draw (v6) -- (v12) ;
    \draw (v7) -- (v13) ;
    \draw (v7) -- (v14) ;
    \draw (v8) -- (v15) ;
    \draw (v8) -- (v16) ;

\end{tikzpicture}
\caption{Tree corresponding to clause 3}
\end{subfigure}

\vspace{0.2in}

\begin{subfigure}[t]{\textwidth}
\begin{center}
\begin{tikzpicture}[scale=0.6,
    node/.style={%
      draw,
      circle,
    },
  ]

    \node [node] (v0) {$f_5$};
    \path (v0) ++(-150:1.5\nodeDist) node [node] (v1) {};
    \path (v0) ++(-30:1.5\nodeDist) node [node] (v2) {$f_6$};
    \path (v2) ++(-150:1.5\nodeDist) node [node] (v3) {};
    \path (v2) ++(-30:1.5\nodeDist) node [node] (v4) {$f_7$};
    \path (v4) ++(-150:1.5\nodeDist) node [node] (v7) {};
    \path (v4) ++(-30:1.5\nodeDist) node [node] (v8) {};

    \path (v1) ++(270:0.1\nodeDist) node [below] {$12$};
    \path (v3) ++(270:0.1\nodeDist) node [below] {$6$};
    \path (v7) ++(270:0.1\nodeDist) node [below] {$0$};
    \path (v8) ++(270:0.1\nodeDist) node [below] {$1$};

    \draw (v0) -- (v1) ;
    \draw (v0) -- (v2) ;
    \draw (v2) -- (v3) ;
    \draw (v2) -- (v4) ;
    \draw (v4) -- (v7) ;
    \draw (v4) -- (v8) ;

\end{tikzpicture}
\end{center}
\caption{Auxiliary tree}
\end{subfigure}
\caption{Example of (\probname) instance originating from \eqref{eq:3satex}. The left branches correspond to setting the feature to 0. The right ones set the feature to 1. The numbers below are the $\nodesupport[\clause,\node,\class]$ values.}
\label{fig:3satex}
\end{figure}

Theorem~\ref{thm:nphard} follows from Claims~\ref{cl:3sat2drp} and \ref{cl:drp23sat}.

\begin{claim}
If the 3-SAT instance is a YES instance, then (\probname) is feasible. 
\label{cl:3sat2drp}
\end{claim}
\begin{proof}
Suppose that there is an assignment of values to the $\nsatvars$ variables of 3-SAT so that all clauses $\clauses$ are satisfied. 

Then for each clause $\clause\in \clauses$, pick the leaf $\node$ of the perfect binary subtree that corresponds to the assignment of variables in 3-SAT. We set $\varyb[\clause,\node,(6|\clauses|+1),]=1$ and  
$\varyb[\clause,\node',(6|\clauses|+1),]=0$ for all $\node'\in \leaves[\clause]$, $\node'\neq \node$. We set the first  $\nsatvars$ attributes of $x_{6|\clauses|+1}$ to match the assignment of the $\nsatvars$ variables that satisfy 3-SAT. We set the attributes $\nsatvars+1,\ldots,\nsatvars+|\clauses|$ to 1. 

For the remaining examples, we set the attributes $\nsatvars+\clause$ of examples $6(\clause-1)+1, \ldots, 6 \clause $ to 1, and all other attributes in $\nsatvars+1,\ldots,\nsatvars+|\clauses|$ to 0. This is done for all $l=1,\ldots,|\clauses|$. In addition, we set the remaining attributes to match one of the leaves that have $\nodesupport[\clause,\node,\class]=1$ but don't correspond to the assignment of variables in 3-SAT and set the corresponding  $\varyb[\clause,\node,\example,]$ to 1.

This can be easily checked to be feasible for (\probname).

For example, in Figure~\ref{fig:3satex}, if the 3-SAT assignment is $u_1=\textsc{True}$, $u_2=u_3=u_4=\textsc{False}$, then we would have the solution to (\probname) shown in Table~\ref{tab:3sat2drp}.

\begin{table}[h!]
\begin{center}
\begin{tabular}{l|llll|lllllllllllllllllllllllllllll}
      & $f_1$ & $f_2$ & $f_3$ & $f_4$ & $f_5$ & $f_6$ & $f_7$ \\ \hline
$x_1$ &  0    &  0    &   0   & {\bf 1} & 1   &  0    &   0   \\
$x_2$ &  0    &  0    &   1   & {\bf 1} & 1   &  0    &   0   \\
$x_3$ &  0    &  1    &   0   & {\bf 1} & 1   &  0    &   0   \\
$x_4$ &  1    &  0    &   1   & {\bf 1} & 1   &  0    &   0   \\
$x_5$ &  1    &  1    &   0   & {\bf 1} & 1   &  0    &   0   \\
$x_6$ &  1    &  1    &   1   & {\bf 1} & 1   &  0    &   0   \\ \hline
$x_7$    &  0    &  0    & {\bf 1} &   1   & 0   &  1    &   0   \\
$x_8$    &  0    &  1    & {\bf 1} &   0   & 0   &  1    &   0   \\
$x_9$    &  0    &  1    & {\bf 1} &   1   & 0   &  1    &   0   \\
$x_{10}$ &  1    &  0    & {\bf 1} &   1   & 0   &  1    &   0   \\
$x_{11}$ &  1    &  1    & {\bf 1} &   0   & 0   &  1    &   0   \\
$x_{12}$ &  1    &  1    & {\bf 1} &   1   & 0   &  1    &   0   \\ \hline
$x_{13}$ & {\bf 1} &  0    &  0    &   1   & 0   &  0    &   1   \\
$x_{14}$ & {\bf 1} &  0    &  1    &   0   & 0   &  0    &   1   \\
$x_{15}$ & {\bf 1} &  0    &  1    &   1   & 0   &  0    &   1   \\
$x_{16}$ & {\bf 1} &  1    &  0    &   0   & 0   &  0    &   1   \\
$x_{17}$ & {\bf 1} &  1    &  0    &   1   & 0   &  0    &   1   \\
$x_{18}$ & {\bf 1} &  1    &  1    &   0   & 0   &  0    &   1   \\ \hline
$x_{19}$ &      \emph{1}  &  \emph{0}    &  \emph{0}    &   \emph{0}   & 1   &  1    &   1
\end{tabular}
\end{center}
\caption{Solution to (\probname) constructed from 3-SAT solution. The entries in \textbf{bold} are arbitrary. The \emph{italicized} entries encode the solution of the corresponding 3-SAT problem.}
\label{tab:3sat2drp}
\end{table}
\end{proof}

\pagebreak
\begin{claim}
If (\probname) is feasible then the 3-SAT instance is a YES instance. 
\label{cl:drp23sat}
\end{claim}
\begin{proof}
If (\probname) is feasible, then there exists one example which has features $\nsatvars+1,\ldots,\nsatvars+|\clauses|$ equal to 1. This comes from the rightmost node of the auxiliary tree. 

Without loss of generality, assume that such example is $x_{6|\clauses| + 1}$ (\emph{e.g.,} $x_{19}$ in Table~\ref{tab:3sat2drp}).

Then in each of the trees $\clause \in \clauses$, $x_{6|\clauses| + 1}$ must have gone to the right branch at the root. In this case, we know that $x_{6|\clauses| + 1}$ must fall into one of the leaves of the perfect binary tree that corresponds to a truth assignment that makes the clause $\clause$ satisfied. 

A solution to 3-SAT can then be constructed by looking at the first $\nsatvars$ components of $x_{6|\clauses| + 1}$.
\end{proof}

\section{Mixed-Integer Linear Programming Formulation}\label{appendix:milp_models}

We show how the reconstruction problem can be alternatively formulated as a Mixed-Integer Linear Program (MILP), permitting the use of alternative solution algorithms. In a MILP, all variables can be continuous or integers, but all constraints and the (optional) objective function must be linear in the decision variables. This restriction is not imposed in Constraint Programming (CP). Consequently, we must \emph{linearize} some of the expressions required to model our reconstruction problem using additional variables. We describe the MILP formulation for the scenario where bagging is not used to train the target random forests, before performing some empirical evaluation of its performance.

\subsection{Model Formulation (Without Bagging)}
\label{subsec:milp_formulation}
\def\alldepth[#1]{\mathcal{D}_{#1}}
\def\nodesatdepth[#1,#2]{\mathcal{V}_{#1#2}^{I}}

\def\fixedfeat[#1]{\mathcal{F}_{#1}}

\def\soly[#1,#2]{x_{#1#2}^{\varyb[,,,]}}

\def\fixnode[#1,#2]{\phi^{\varyb[,,,]}_{#1#2}}

\def\varw[#1,#2,#3]{w_{#1#2#3}}

\def\vareta[#1,#2]{\eta_{#1#2}}

\def\lastidxclset[#1]{\ell_{#1}}

We present here a MILP model for the \probname.
Our MILP model for reconstructing the training set of a given random forest extends the OCEAN framework~\citep{parmentier2021optimal}, which was proposed to generate optimal counterfactual explanations for tree ensembles. 
In a nutshell, OCEAN leverages MILP to encode the structure of the trees within the forest, and aims at finding an example as close as possible from a query example $\attributes[\example]$ but with a different classification. Rather than determining the attributes' vector of a single example (the generated counterfactual), we aim to reconstruct the features' vector of all the $\nexamples$ training examples simultaneously.

We now introduce some additional notation. For each tree $\tree \in \forest$, we define $\alldepth[\tree]$ as the set of all the depths reached in $\tree$, and $\forall \depth \in \alldepth[\tree]$, $\nodesatdepth[\tree,\depth]$ is the set of internal nodes at depth $\depth$ in $\tree$.
As mentioned in Section~\ref{sec:method_implementation}, without bagging, one can fix in advance the set of decisions $\varz[\example,\class]$ (\emph{i.e.,} if an example $\example$ is from class $\class$). 
Let $\classset[\class] = \{ \example \in \{1..\nexamples\}:\varz[\example,\class]=1\}$ be the set of indices of examples belonging to class $\class$, and $\nodesforfeat[\tree, \feature]$ be the set of nodes within tree $\tree$ splitting on feature $\feature$.
W.l.o.g., we assume that the indices in $\classset[\class]$ are consecutive.

We first define decision variables that will model the path of each example through each tree:
\begin{itemize}
    \item $\forall \tree \in \forest, \forall \depth \in \alldepth[\tree], \forall \example \in \{1..\nexamples\}$ : $\varlambda[\tree,\depth,\example] \in \{0,1\}$ takes value $1$ if example $\example$ takes the left path at depth $\depth$ of the tree $\tree$, and $0$ otherwise. The value is free if the path doesn't go this deep.
    \item $\forall \tree \in \forest, \forall \node \in \internalnodes[\tree]\bigcup \leaves[\tree], \forall \example \in \{1..\nexamples\}$: $\vary[\tree,\node,\example] \in [0;1]$ takes value $1$ if example $\example$ reaches node $\node$ of the tree $\tree$, $0$ otherwise (note that the integrality is forced by the previous variables)
    \item $ \forall \example \in \{1..\nexamples\}, \forall \feature \in \{1..\nattributes\}$: $\varx[\example,\feature] \in \{ 0;1 \}$ is the value of feature $\feature$ for example $\example$ in the reconstruction
\end{itemize}

First, the following constraints correspond to the one-hot encoding of the features:
\begin{flalign}
\setstackgap{L}{10pt}
    & \sum\limits_{\feature \in w} \varx[\example,\feature] = 1 & & \forall \example \in \{1..\nexamples\}, \forall \ohegroup \in \ohevects \nonumber
\end{flalign}

We then use the following constraints to model the flow of the examples through the trees:
\begin{flalign}
    & \vary[\tree,1,\example] = 1 & & \forall \tree \in \forest, \forall \example \in \{1..\nexamples\} \\
    & \vary[\tree,\node,\example] = \vary[\tree,\leftchild[\node],\example] + \vary[\tree,\rightchild[\node],\example] & & \forall \tree \in \forest, \forall \node \in \internalnodes[\tree], \forall \example \in \{1..\nexamples\}\\
    & \sum\limits_{\node \in \nodesatdepth[\tree,\depth]} \vary[\tree,\leftchild[\node],\example] \leq \varlambda[\tree,\depth,\example] & & \forall \tree \in \forest, \forall \depth \in \alldepth[\tree], \forall \example \in \{1..\nexamples\}\\
    & \sum\limits_{\node \in \nodesatdepth[\tree,\depth]} \vary[\tree,\rightchild[\node],\example] \leq 1 - \varlambda[\tree,\depth,\example] & & \forall \tree \in \forest, \forall \depth \in \alldepth[\tree], \forall \example \in \{1..\nexamples\}
\end{flalign}

In a nutshell, because we consider the case without the use of bagging, each example has one associated unit of flow at the root of each tree. This flow is encoded by continuous variables (which are easier to handle for the solver than integer/binary ones). All the flow is then directed through the tree, by going either left or right at each split node, until it reaches a leaf. 

We then link these flows to the values taken by the features of the examples through the following constraints:
\begin{flalign}
\setstackgap{L}{10pt}
    & \varx[\example,\feature] \leq 1-\vary[\tree,\leftchild[\node],\example] & & {\forall \example \in \{1..\nexamples\}, \forall \feature \in \{1..\nattributes\}, \forall \tree \in \forest, \forall \node \in \nodesforfeat[\tree,\feature]}\\ 
    & \vary[\tree,\rightchild[\node],\example] \leq \varx[\example,\feature] & &  
    \forall \example \in \{1..\nexamples\}, \forall \feature \in \{1..\nattributes\}, \forall \tree \in \forest, \forall \node \in \nodesforfeat[\tree,\feature]
\end{flalign}

Finally, we connect these flows to the support of each node within the trees (recall that because we consider the case without bagging, $\oheclass[\class,\example]$ is a prefixed constant, and hence the computation is linear in the decision variables $\vary[\tree,\node,\example]$):
\begin{flalign}
\setstackgap{L}{10pt}
    & \nodesupport[\tree,\node,\class] = \sum\limits_{{\example \in \{1..\nexamples\}}} \vary[\tree,\node,\example] \oheclass[\class,\example], & & \forall \tree \in \forest, \forall \node \in \internalnodes[\tree]\bigcup \leaves[\tree], \forall \class \in \classes
\end{flalign}

Note that we additionally use the following constraints for symmetry breaking in each class:
\begin{flalign}
    & \sum\limits_{\feature \in \{1..\nattributes\}} 2^{\feature-1}\varx[\example,\feature] \leq \sum\limits_{\feature \in \{1..\nattributes\}}2^{\feature-1}\varx[(\example+1),\feature] & & \forall \class \in \classes, \forall \example \in \classset[\class] \setminus \{\lastidxclset[\class]\}
\end{flalign}
where $\lastidxclset[\class]$ is the last index in $\classset[\class]$.

In the next subsection, we empirically evaluate our proposed MILP model and compare it to the CP formulation introduced in the main paper. Note that extending the proposed MILP to handle bootstrap sampling is possible, but the number of required variables increases prohibitively in order to preserve linearity, limiting the scalability of the approach.

\subsection{Empirical Evaluation}

We run the reconstruction experiments on the COMPAS dataset without bagging as described in Section~\ref{subsec:expes_setup}, using our MILP formulation, and compare the results with those obtained using our CP model (which are reported in Section~\ref{subsec:results}). 
The MILP models are solved using the \gurobi{} solver~\citep{gurobi} through its Python binding\footnote{\url{https://pypi.org/project/gurobipy/}}, all the other experimental parameters remaining unchanged.

The results are reported in Figure~\ref{fig:results_milp_vs_cp_compas}, and their run times are compared in Table~\ref{tab:runtimes_compas_cp_vs_milp}.
Note that the results for the CP model are those presented in Figure~\ref{fig:results_compas_no_bagging}, repeated here to ease comparison.
Comparing the different curves (which correspond to different maximum depth constraints) between Figures~\ref{fig:results_milp_vs_cp_compas_cp} and~\ref{fig:results_milp_vs_cp_compas_milp}, we see that both approaches successfully solve the dataset reconstruction problem on COMPAS without the use of bagging to train the target random forests. Intuitively, the two feasibility models encode the same information, and define the same set of feasible reconstructions. Because they use different techniques to represent and explore it, they may end up with different reconstructions, but there is no \emph{a priori} reason for one to outperform the other systematically, and as observed in our experiments, their reconstruction performances are generally similar. 

Nevertheless, Table~\ref{tab:runtimes_compas_cp_vs_milp} highlights significant solution-time differences between the CP and MILP approaches. The solution times of both approaches are of the same order of magnitude for shallow trees. However, as the depth of the trees grows, the solution time increases more quickly with the MILP than with the CP model. 
For instance, on average, the MILP formulation requires over three times more CPU time than the CP one when no maximum depth constraint is set. More importantly, the solution times are considerably less stable when using the MILP, resulting in larger maximum run times. In the most extreme case, the MILP exceeds 75 minutes, contrasting sharply with the CP model's consistently modest durations, never surpassing three minutes. As discussed in the previous subsection, the MILP is also less prone to be extended to the setup where bagging is used to train the target random forests. These observations led us to rely on the CP model in the main paper.

\begin{figure*}[]
  \centering
  \begin{subfigure}[t]{0.48\textwidth}
    \centering\includegraphics[width=\textwidth]{compas_cp-sat_bagging=False_average_acc.pdf}
    \caption{CP model}\label{fig:results_milp_vs_cp_compas_cp}
  \end{subfigure}
  \begin{subfigure}[t]{0.48\textwidth}
    \centering\includegraphics[width=\textwidth]{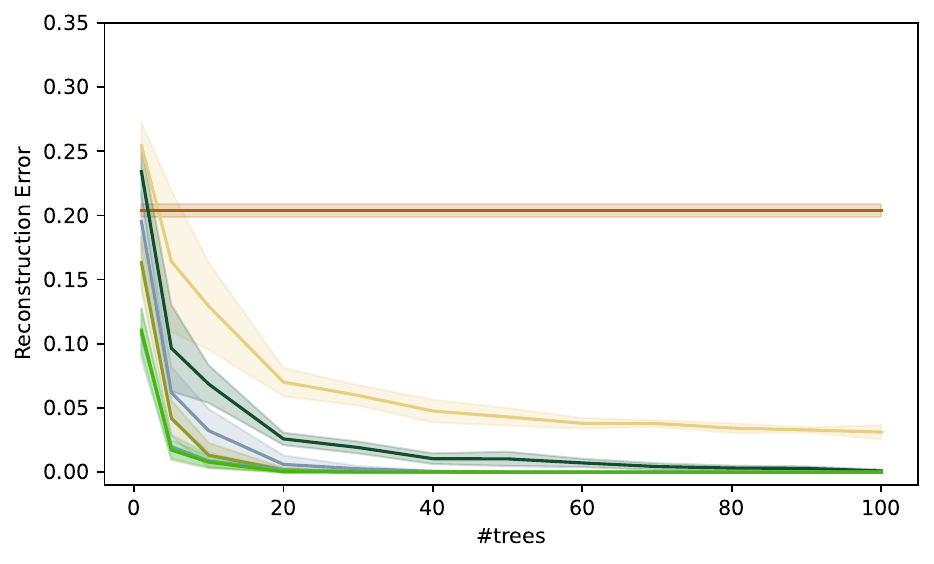}
    \caption{MILP model}\label{fig:results_milp_vs_cp_compas_milp}
  \end{subfigure}
  
  \includegraphics[width=0.75\textwidth]{average_acc_legend.pdf}
    \caption{Average reconstruction error as a function of the number of trees $\lvert \forest \rvert$ within the attacked forest $\forest$, for different maximum depth values $\depth_{max}$ and for the random baseline. For the experiments on the COMPAS dataset, not using bagging, we report the results obtained using either the CP model (Section~\ref{sec:method_implementation}) or the MILP one (Section~\ref{subsec:milp_formulation}).}
    \label{fig:results_milp_vs_cp_compas}
\end{figure*}
\begin{table}[]
\centering
\caption{Reconstruction times for the experiments on the COMPAS dataset, without the use of bagging to train the target random forests. For both the CP and the MILP models, we consider all the measured runtimes (\emph{i.e.,} for the $5$ seeds and the $12$ different numbers of trees within the forests) for a given maximum depth constraint. We report their average value, along with the standard deviation, the minimum time and the maximum one.}\label{tab:runtimes_compas_cp_vs_milp}
\begin{tabular}{@{}cccccc@{}}
\toprule
\multirow{2}{*}{\textbf{Max. Depth}} & \multirow{2}{*}{\textbf{Method}} & \multicolumn{4}{c}{\textbf{Reconstruction Times (s)}}     \\ \cmidrule(l){3-6} 
                                     &                                  & \textbf{Avg} & \textbf{Std} & \textbf{Min} & \textbf{Max} \\ \midrule\midrule
\multirow{2}{*}{2}                   & CP                               & 4.7          & 3.8          & 0.1          & 17.0         \\ \cmidrule(l){2-6} 
                                     & MILP                             & 5.4          & 6.1          & 0.1          & 31.2         \\ \midrule\midrule
\multirow{2}{*}{3}                   & CP                               & 14.1         & 12.9         & 0.1          & 48.4         \\ \cmidrule(l){2-6} 
                                     & MILP                             & 10.3         & 21.9         & 0.2          & 162.5        \\ \midrule\midrule
\multirow{2}{*}{4}                   & CP                               & 25.2         & 18.9         & 0.2          & 58.9         \\ \cmidrule(l){2-6} 
                                     & MILP                             & 24.7         & 52.6         & 0.2          & 302.7        \\ \midrule\midrule
\multirow{2}{*}{5}                   & CP                               & 34.3         & 23.4         & 0.3          & 85.8         \\ \cmidrule(l){2-6} 
                                     & MILP                             & 26.4         & 60.4         & 0.3          & 418.6        \\ \midrule\midrule
\multirow{2}{*}{10}                  & CP                               & 53.7         & 39.2         & 0.5          & 160.6        \\ \cmidrule(l){2-6} 
                                     & MILP                             & 77.6         & 312.7        & 2.6          & 2471.2       \\ \midrule\midrule
\multirow{2}{*}{None}                & CP                               & 49.7         & 35.5         & 0.6          & 142.0        \\ \cmidrule(l){2-6} 
                                     & MILP                             & 188.3        & 791.2        & 3.3          & 4521.8       \\ \bottomrule
\end{tabular}
\end{table}

\section{Implementation Details for the CP Model}
\label{appendix:cpdetails}

As mentioned in Section~\ref{sec:method_implementation}, in order to implement the CP model using the \texttt{OR-Tools} CP-SAT solver, some other variables and constraints are needed due to the specificities of the solution software. We discuss these technical aspects in this appendix. 

First, CP-SAT only allows implication constraints with a literal being the cause of the implication. Therefore, we had to rely on auxiliary binary variables:
\begin{itemize}
\item For all $\tree \in \forest$, $\node \in \leaves[\tree]$, $k\in\{1..\nexamples\}$: $\varw[\tree,\node,\example]$ is 1 if example $\example$ is classified by leaf $\node$ of tree $\tree$; 0 otherwise.
\end{itemize}

Second, the relationship between the $\varq[\tree,\example,\val]$ and $\varyb[\tree,\node,\example,\class]$ also cannot be enforced directly. It can only be done via another set of auxiliary variables:

\begin{itemize}
\item $\vareta[\tree,\example]\in \Z_+$ represents the number of times example $\example$ is used in tree $\tree$
\end{itemize}

To model the relationship between $w$ and $y$, we add the constraints:
\begin{itemize}
\item \textbf{if } $\varw[\tree,\node,\example]=0$ \textbf{ then } $\varyb[\tree,\node,\example,\class]=0$, $\forall \tree \in \forest, \node \in \leaves[\tree], \example \in \{1..\nexamples\}, \class\in \classes$
\item \textbf{if } $\varw[\tree,\node,\example]=1$ \textbf{ then } $\sum\limits_{\class\in \classes} \varyb[\tree,\node,\example,\class]\geq 1$, $\forall \tree \in \forest, \node \in \leaves[\tree], \example \in \{1..\nexamples\}$
\end{itemize}
These are explicitly added in CP-SAT using the \texttt{OnlyEnforceIf} function that allows a linear constraint only to be enforced if a boolean variable is \textsc{True}.

With these variables, the constraints that were presented before as  
 $$\textbf{if } \sum\limits_{\class \in \classes}\varyb[\tree,\node,\example,\class] \geq 1 \textbf{ then} \quad \left(\bigwedge\limits_{\feature \in \positivesplits[\tree,\node]}{\varx[\example,\feature] = 1}\right) \land \left(\bigwedge\limits_{\feature \in \negativesplits[\tree,\node]}{\varx[\example,\feature] = 0}\right)$$ will now be implemented as:

\begin{itemize}
 \item $\forall \tree \in \forest, \forall \example \in \{1..\nexamples\}, \forall \node \in  \leaves[\tree]:$
 $\textbf{if } \varw[\tree,\node,\example]=1  \textbf{ then} \quad \left(\bigwedge\limits_{\feature \in \positivesplits[\tree,\node]}{\varx[\example,\feature] = 1}\right) \land \left(\bigwedge\limits_{\feature \in \negativesplits[\tree,\node]}{\varx[\example,\feature] = 0}\right)$
\end{itemize}

These also can be explicitly added in CP-SAT using the \texttt{OnlyEnforceIf} function.

To model the correct relationship between $\varyb[\tree,\node,\example,\class]$ and $\vareta[\tree,\example]$ variables, we add the constraints:

\begin{itemize}
    \item For all $\example\in\{1..\nexamples\}$, $\tree\in \forest$: $\vareta[\tree,\example]=\sum\limits_{\class\in \classes} \sum\limits_{\node\in \leaves[\tree]} \varyb[\tree,\node,\example,\class]$
\end{itemize}

Now, the constraints
$$\sum\limits_{\node \in \leaves[\tree],\class \in \classes} \varyb[\tree,\node,\example,\class] = \val \iff \varq[\tree,\example,\val]=1$$
can be implemented in CP-SAT using the constraints
\begin{itemize}
    \item For all $\tree\in \forest$, $\example\in \{1..\nexamples\}$: $\texttt{AddMapDomain}(\vareta[\tree,\example],[\varq[\tree,\example,\val]]_{\val\in \values})$
\end{itemize}

These constraints receive the integer variable $\vareta[\tree,\example]$ and the vector of binary variables $[\varq[\tree,\example,\val]]_{\val\in \values}$ and enforce that $\vareta[\tree,\example]=\val$ if and only if $\varq[\tree,\example,\val]=1$.

\section{Extending the CP Model to Handle Non-Binary Attributes}
\label{appendix:other_types_of_attributes}

The Constraint Programming (CP) model presented in Section~\ref{sec:method_implementation} is able to reconstruct binary attributes. While we focused on this case to streamline the presentation of the methodology and evaluation metrics, our framework can be extended to handle other types of attributes, as explained in this appendix section.

Discrete attributes take values in a finite domain. If these values can be ordered, the attribute is coined as \textbf{ordinal}, and if they can not (\emph{i.e.,} if they represent categories), it is called \textbf{categorical}. These two types of discrete attributes can be handled by \draft{} as detailed hereafter. 

\subparagraph{Categorical Attributes.} Because the different possible values of a categorical attribute can not be ordered, it wouldn't make sense to verify whether they are greater or smaller than a given split value, even if the different categories can be represented using different integer values. Indeed, such attributes must usually be one-hot encoded (\emph{i.e.,} with one separate binary attribute for each possible category, all the created binary attributes summing up to one) and are hence directly and efficiently handled using the formulation described in Section~\ref{sec:method_implementation}.

\paragraph{Ordinal Attributes.} Ordinal attributes can be used directly in tree ensembles (without one-hot encoding) since an order relation permits defining meaningful splits. They can be handled naturally using \draft{}. More precisely, using the CP formulation provided in Section~\ref{sec:method_implementation}, the reconstruction variables $\{\varx[\example,\feature]\}_{\example \in \{1..\nexamples\}}$ associated to each ordinal attribute $\feature$ must be declared as integers (which are directly supported in Constraint Programming). Furthermore, the constraint enforcing the conditions associated to a branch leading to a leaf $\node$ if an example $\example$ is assigned to that leaf in tree $\tree$ must be slightly generalized. We now define $\positivesplits[\tree,\node]$ as the set of attribute-value tuples $(\feature, \attrvalue)$ such that attribute $\feature$ must be greater than $\attrvalue$ for an example to fall into leaf $\node$. Similarly, $\negativesplits[\tree,\node]$ is now the set of attribute-value tuples $(\feature, \attrvalue)$ such that attribute $\feature$ must be smaller or equal to $\attrvalue$ for an example to fall into $\node$. Note that this slight generalization also encompasses the binary attribute case, where the split value $\attrvalue$ is usually fixed to $0.5$. The generalized constraint then becomes:

$$\forall \tree \in \forest, \forall \example \in \{1..\nexamples\}, \forall \node \in  \leaves[\tree]:
\textbf{if } \sum\limits_{\class \in \classes}\varyb[\tree,\node,\example,\class] \geq 1 \textbf{ then} \quad \left(\bigwedge\limits_{(\feature, \attrvalue) \in \positivesplits[\tree,\node]}{\varx[\example,\feature] > \attrvalue}\right) \land \left(\bigwedge\limits_{(\feature, \attrvalue) \in \negativesplits[\tree,\node]}{\varx[\example,\feature] \leq \attrvalue}\right)$$

All the other variables and constraints remaining unchanged, the model provided in Section~\ref{sec:method_implementation} can effectively be used to reconstruct discrete (categorical or ordinal) features.

Contrary to discrete attributes, \textbf{numerical} ones take values in a continuous space. If the underlying mathematical programming framework can encode continuous variables (which is, for instance, the case of Mixed-Integer Linear Programming), then numerical attributes can be handled just like ordinal ones, using the methodology described in the previous paragraph. This is not the case in Constraint Programming, but numerical attributes can still be reconstructed effectively, as discussed hereafter.

\paragraph{Numerical Attributes.} While numerical attributes take values in a continuous space, the number of nodes within a decision tree (hence within a random forest) is finite, and so the number of split values regarding any specific attribute is also finite. Then, the number of possible values or intervals for a given reconstructed numerical attribute is also discrete. Indeed, the knowledge acquired from a random forest can indicate that an example's numerical attribute lies within a given interval (between two split values), but in the general case, it does not indicate which particular value within this interval it should take. We leverage such discretization to reconstruct numerical attributes as follows:
\begin{enumerate}
    \item We parse all the trees in the forest and build the \underline{ordered} list of all the different split values regarding each numerical feature $\feature$:
    $$ \splitvalues[\feature] = \texttt{sorted}\left(\left\{\attrvalue : (\feature, \attrvalue) \in (\positivesplits[\tree,\node]\cup \negativesplits[\tree,\node])_{\tree \in \forest, \node \in  \leaves[\tree]}  \right\}\right)$$

    \item We concatenate this ordered list of split values to the (possibly infinite) lower and upper bounds on the domain of attribute $\feature$:
    $$ \intervalvalues[\feature] = \{ \texttt{lower\_bound}(\feature) \} \cup \splitvalues[\feature] \cup \{ \texttt{upper\_bound}(\feature) \}$$
    Intuitively, $\intervalvalues[\feature]$ defines the possible intervals for attribute $\feature$ given the splits within the forest.
    
    \item To build the reconstruction model, we encode each numerical feature $\feature$ as an ordinal (integer) one $\feature'$, taking values in $\{1..(\lvert \splitvalues[\feature] \rvert + 1)\}$. 
    The value of ordinal attribute $\feature'$ in the reconstruction performed by the CP model (variables $\{\varx[\example,\feature']\}_{\example \in \{1..\nexamples\}}$) will then be used to retrieve the interval in which numerical feature $\feature$ lies.
    Note that the first $\lvert \splitvalues[\feature] \rvert$ values correspond to the different split values for attribute $\feature$ while the $(\lvert \splitvalues[\feature] \rvert + 1)$ one encodes the situation where $\feature$ is strictly greater than its largest split value. Note that since the constraints associated to the splits are either ``strictly greater than" or ``smaller or equal to", it is not possible to forbid the smallest split value, and so we do not need to insert an additional value before it. 
    
    \item For each split-value tuple $(\feature, \attrvalue)$ associated to numerical feature $\feature$ in $(\positivesplits[\tree,\node]\cup \negativesplits[\tree,\node])$, we create for the corresponding integer feature $\feature'$ a split-value tuple $(\feature', \attrvalue')$ such that $\attrvalue'$ is the index of $\attrvalue$ in the ordered list $\splitvalues[\feature]$. Using such split-value tuple, attribute $\feature'$ can then be handled just like other ordinal features, as aforementioned.
    
    \item Once the reconstruction is done, we have to connect the value of $\feature'$ to that of the actual (continuous) numerical feature $\feature$. If for $\example \in \{1..\nexamples\}$, $\varx[\example,\feature'] = \attrvalue' \in \{1..(\lvert \splitvalues[\feature] \rvert + 1)\}$ in the reconstruction performed by the CP model, we set the value of the corresponding reconstructed (continuous) attribute to the mean between the $\attrvalue'-1$ and $\attrvalue'$ split values (\emph{i.e.,} to $\frac{\intervalvalues[\feature]{[}\attrvalue'{]}+\intervalvalues[\feature]{[}\attrvalue'+1{]}}{2}$ - the difference in indices comes from the fact that $\intervalvalues[\feature]$ starts with an additional element, corresponding to the attribute's lower bound). Note that if the lower and (or) upper bounds of $\feature$ are infinite, we can choose any arbitrary value compatible with the splits' information.

\end{enumerate}

\section{The Impact of Bagging on Data Protection}
\label{appendix:benchmark}

Our results (reported in Section~\ref{sec:experiments}) show that if bagging is not used, then all of the data can be recovered with just a few trees in the RF. 
However, with bagging, the CP model from  Section~\ref{methodology} can recover around 90-95\% of the data, even with many trees. 
In this appendix, we present experiments designed to understand why we could not recover 100\% of the data with bagging. 

One of the complicating aspects of bagging is that the knowledge of how many times a sample has been classified within a given leaf $\node$ of a given tree $\tree$ is lost.
With this in mind, we posed the following question: 
\begin{itemize}
\item Considering the CP model from Section~\ref{methodology}, if we know in advance the values 
of $\varyb[\tree,\node,\example,\class]$ (that is, how many times sample $\example$ is classified within a given leaf $\node$ of tree $\tree$ as part of class $\class$) 
how much reduction can be observed in the reconstruction error?
\end{itemize}

Note that, if the values of $\varyb[\tree,\node,\example,\class]$ are given, then the values of $\varz[\example,\class]$ and $\varq[\tree,\example,\val]$ can be deducted.
So the only remaining issue is to determine the
$\varx[\example,\feature]$ values and the only constraints that need to be enforced on those are the one-hot encoding constraints and the \emph{leaf-consistency} constraints:
\begin{itemize}
 \item $\forall \tree \in \forest, \forall \example \in \{1..\nexamples\}, \forall \node \in  \leaves[\tree]:$
 $\textbf{if } \sum\limits_{\class \in \classes}\varyb[\tree,\node,\example,\class] \geq 1 \textbf{ then} \quad \left(\bigwedge\limits_{\feature \in \positivesplits[\tree,\node]}{\varx[\example,\feature] = 1}\right) \land \left(\bigwedge\limits_{\feature \in \negativesplits[\tree,\node]}{\varx[\example,\feature] = 0}\right)$
\end{itemize}

Let $\fixnode[\tree,\example]:=\left\{ \node \in  \leaves[\tree] : \sum\limits_{\class \in \classes}\varyb[\tree,\node,\example,\class] \geq 1\right\}$  be the (possibly empty) set of leaves of tree $\tree$ for which example $\example$ has been used. While the leaf-consistency constraints fix all attributes $\feature$ in $\positivesplits[\tree,\node]\cup \negativesplits[\tree,\node]$ for $v\in \fixnode[\tree,\example]$, any feature that does not appear in any such sets (call them \emph{free attributes}) can be arbitrarily set without changing the likelihood of the solution. And so, the fact that a free attribute is guessed correctly can be attributed to luck and should not be seen as a positive aspect of the CP model. 

Formally, the fixed attributes for example $\example\in \{1..\nexamples\}$ are 
$$\fixedfeat[\example]:=\bigcup\limits_{\tree \in \forest} \bigcup\limits_{\node \in \fixnode[\tree,\example]} \left(\positivesplits[\tree,\node]\cup\negativesplits[\tree,\node]\right)$$
and the free attributes are $\bar{\fixedfeat[\example]}:=\{1..\nattributes\} \setminus \fixedfeat[\example].$

Let ${\{\attributes[\example];\class_{\example}\}}^{\nexamples}_{\example=1}$ be the training set which was used to train the random forest (and which we are trying to recover). 
With this we define $\soly[\example,]$ as follows:
$$\soly[\example,\feature]:=
\begin{cases}
1, & \mbox{if } i\in \fixedfeat[\example]\cap \positivesplits[\tree,\node] \mbox{ for some } \tree \in \forest, v\in \fixnode[\tree,\example] \\
0, & \mbox{if } i\in \fixedfeat[\example]\cap \negativesplits[\tree,\node] \mbox{ for some } \tree \in \forest, v\in \fixnode[\tree,\example] \\
1 - \attributes[\example\feature], & \mbox{ otherwise}
\end{cases}
$$
$\soly[\example,]$ can be thought of as the solution that is consistent with the $\varyb[,,,]$ variables on all fixed attributes and incorrectly guesses the values of all free attributes, so the worst possible solution that is consistent with $\varyb[,,,]$.

Our \emph{benchmark} experiment can now be described as follows:
\begin{itemize}
\item Run the CP model of Section~\ref{methodology} with $\varx[\example,\feature] = \attributes[\example\feature]$ and $\varz[\example,\class_{\example}]=1$ for all $\example \in\{1..\nexamples\}, \feature\in\{1..\nattributes\}$.
\item Obtain from the solution of such model the values of the $\varyb[\tree,\node,\example,\class]$ variables, for all $\tree\in \forest$, $\node \in \leaves[\tree]$, $\example\in\{1..\nexamples\}$, $\class\in \classes$.
\item Output the set of solutions $\{\soly[\example,]\}_{k=1}^{\nexamples}$.
\end{itemize}
Intuitively, we get the best possible guess for the $\varyb[\tree,\node,\example,\class]$ variables by solving the maximum likelihood problem when the training set is given. Subsequently, we get the worst possible solution that is consistent with that guess. It is worth noting that the knowledge of the training set is used in an advantageous way only to obtain the best possible guess for the $\varyb[\tree,\node,\example,\class]$ variables.

The results of the benchmark experiments for the three considered datasets are shown in Figures~\ref{fig:results_bench_compas}, \ref{fig:results_bench_adult} and~\ref{fig:results_bench_default_credit}. 
The results without bagging are also repeated in Figures~\ref{fig:results_nobag_compas_appendix}, \ref{fig:results_nobag_adult_appendix} and~\ref{fig:results_nobag_default_credit_appendix} (from Figures~\ref{fig:results_compas_no_bagging}, \ref{fig:results_adult_no_bagging} and~\ref{fig:results_default_credit_no_bagging}) for reference and easy comparison. 

The results show that, if one can correctly guess the $\varyb[\tree,\node,\example,\class]$ variables, one can get much closer to recovering 100\% of the data, as in the situation without bagging. Accordingly, the key difficulty in recovering the data is guessing which examples were used in each tree. This corroborates the fact that bagging can help prevent data reconstruction. It also answers the question posed at the beginning of this section. Note that bagging was theoretically shown to intrinsically provide some differential privacy guarantees~\citep{DBLP:conf/ijcai/LiuJG21}, which is consistent with our findings.

It is also interesting to note that the number of trees needed to recover the data without bagging seems to be lower than in the benchmark runs, except for very shallow trees. This makes sense since, without bagging, every tree $\tree$ provides some information about every example $\example$ via the sets $\fixnode[\tree,\example]$, while this is not true with bagging.

One can observe another surprising trend when comparing the curves corresponding to shallow trees (\emph{e.g.,} maximum depth of $2$). Indeed, without bagging, the reconstruction error decreases until a certain value and remains more or less constant, even when increasing the number of trees further. This does not happen in the benchmark runs, and even with very shallow trees, the reconstruction error (which in this experiment is the worst we can expect) converges close to $0$. In fact, a large number of trees trained with bagging seems to provide more information (with the knowledge of the values 
of $\varyb[\tree,\node,\example,\class]$) than the same number of trees trained without bagging. An explanation for this behavior could lie in the trees' intrinsic diversity and in the fact that each of them contains more information about some training samples, namely those that appeared several times in their training data.

\begin{figure}[h!]
  \centering
  \begin{subfigure}[t]{0.48\textwidth}
    \centering\includegraphics[width=\textwidth]{compas_cp-sat_bagging=False_average_acc.pdf}
    \caption{COMPAS dataset, bagging not used}\label{fig:results_nobag_compas_appendix}
  \end{subfigure}  
    \begin{subfigure}[t]{0.48\textwidth}
    \centering\includegraphics[width=\textwidth]{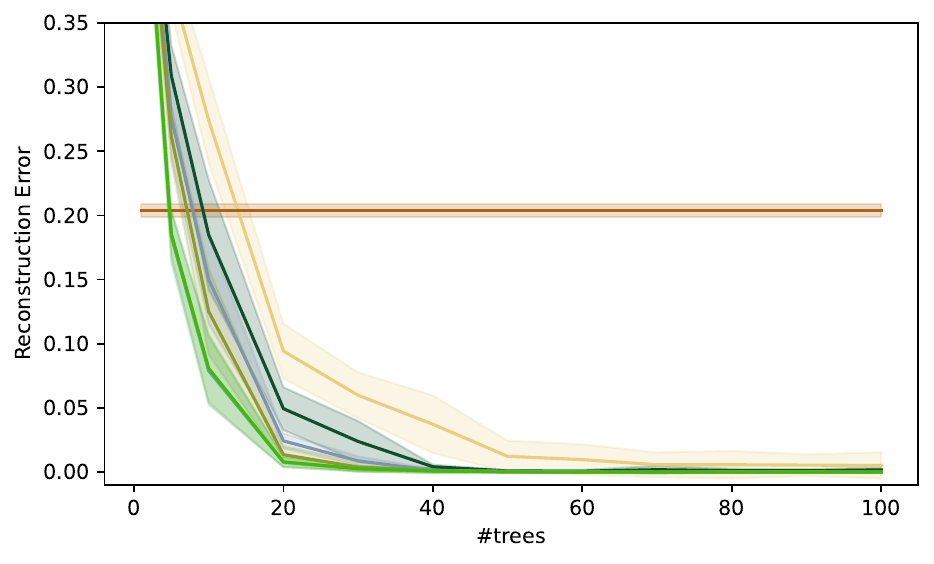}
    \caption{COMPAS dataset, benchmark runs} 
    \label{fig:results_bench_compas}
  \end{subfigure}

  \begin{subfigure}[t]{0.48\textwidth}
    \centering\includegraphics[width=\textwidth]{adult_cp-sat_bagging=False_average_acc.pdf}
    \caption{UCI Adult Income dataset, bagging not used}\label{fig:results_nobag_adult_appendix}
  \end{subfigure}  
      \begin{subfigure}[t]{0.48\textwidth}
    \centering\includegraphics[width=\textwidth]{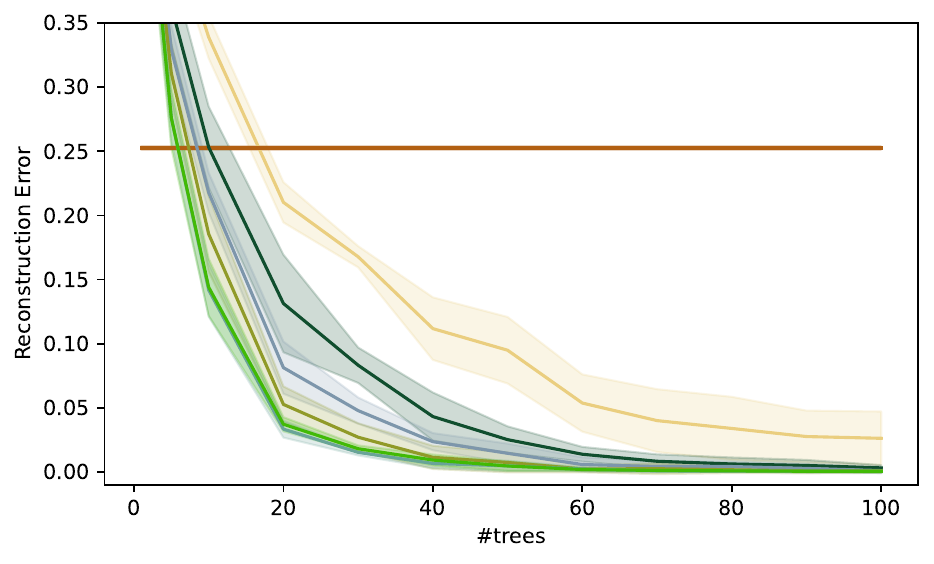}
    \caption{UCI Adult Income dataset, benchmark runs}\label{fig:results_bench_adult}
  \end{subfigure}

      \begin{subfigure}[t]{0.48\textwidth}
    \centering\includegraphics[width=\textwidth]{default_credit_cp-sat_bagging=False_average_acc.pdf}
    \caption{Default of Credit Card Client dataset, bagging not used}\label{fig:results_nobag_default_credit_appendix}
  \end{subfigure}  
    \begin{subfigure}[t]{0.48\textwidth}
    \centering\includegraphics[width=\textwidth]{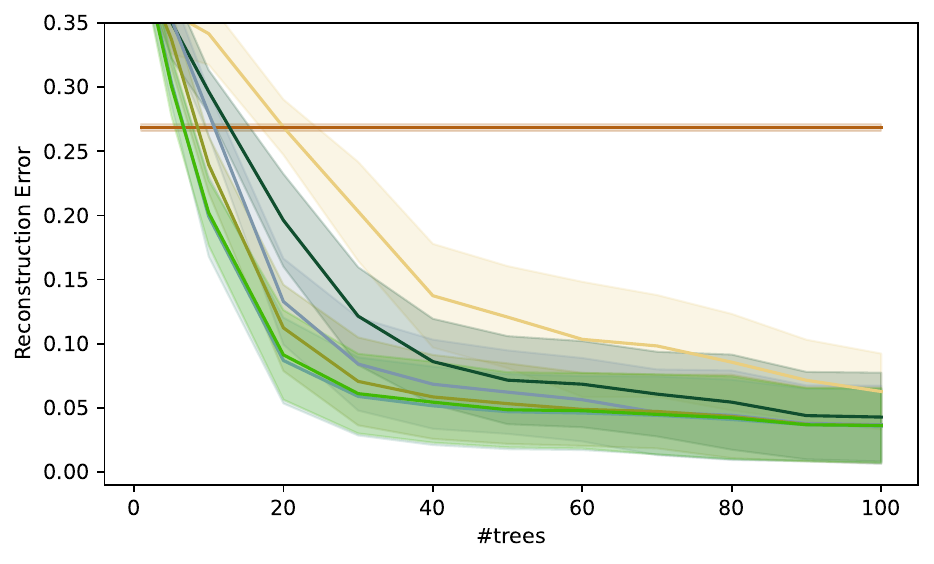}
    \caption{Default of Credit Card Client dataset, benchmark runs}\label{fig:results_bench_default_credit}
  \end{subfigure}

  \includegraphics[width=0.75\textwidth]{average_acc_legend.pdf}
\caption{Comparison of the benchmark results  (using bagging, worst possible reconstruction error using our set of constraints if the number of occurrences of each example within each tree are known) with the ``no-bagging'' ones}\label{fig:results_bench_all}
\end{figure}

\section{Additional Experiments on Scalability}
\label{appendix:scalability_reconstr}

This appendix section aims to investigate how our reconstruction attack performs in terms of both reconstruction error and time when the size $\nexamples$ of the reconstructed training set varies. We additionally investigate whether or not generalization error affects our reconstruction process. To this end, we re-run our experiments without the use of bagging, for the three datasets and \sklearn{}'s default configuration (\emph{i.e.,} with a fixed number of trees $\lvert \forest \rvert = 100$ and no fixed maximum depth). The setup is as described in Section~\ref{subsec:expes_setup}, but we vary the size of the sub-sampled training set between $25$ and $1{,}500$ examples and set the reconstruction time limit to $6$ hours.
The results are provided within Tables~\ref{tab:scalability_compas}, \ref{tab:scalability_adult}, and~\ref{tab:scalability_default_credit} for the three datasets. More precisely, we report for each training set size, the performances (train and test accuracy) of the trained random forests (averaged over the $5$ different random samplings), the reconstruction error, as well as the minimum, maximum, average, and standard deviation of the reconstruction times. As can be seen in the tables, for the UCI Adult Income (respectively, the Default of Credit Card Client) dataset, we report no result for $\nexamples>750$ (respectively, for $\nexamples>500$). This is due to a technical limitation in the solver we use in our experiments. More precisely, the Python wrapper of \ortools{} is limited in the amount of data it can send to its C++ core. This means that large CP models (above 2GB) can not be solved using this wrapper\footnote{This limitation is discussed on the solver's repository: \url{https://github.com/google/or-tools/issues/3861}.}.

\paragraph{Scalability.} We observe in Tables~\ref{tab:scalability_compas}, \ref{tab:scalability_adult}, and~\ref{tab:scalability_default_credit} that the reconstruction error remains very small ($0$ or very close to $0$) for all the considered values of $N$, i.e., the success of our attack is not affected by the size of the reconstructed training set. However, reconstruction time consistently increases with $\nexamples$. This can be explained by the fact that reconstructing more training examples requires exploring a considerably larger search space. Indeed, increasing $\nexamples$ leads to a linear increase in the number of $\{\varx[\example,\feature]\}_{\example \in \{1..\nexamples\},~\feature \in \{1..\nattributes\}}$ and $\{\varyb[\tree,\node,\example,]\}_{\tree \in \forest,~\node \in \leaves[\tree],~\example \in \{1..\nexamples\}}$ variables, but to an exponential increase of the number of possible solutions of the constraint programming model (\emph{i.e.,} search space size). Fortunately, the solution process of the CP solver, using domain reduction and other strategies, does not require examining all the solutions. Empirically, the growth of reconstruction time as a function of  $\nexamples$ is not exponential but polynomial ---approximately quadratic on the considered datasets according to a power-law regression. Another interesting side effect is that larger training sets often lead to deeper trees: while small datasets can be separated using shallow trees, larger ones often require more splits to be performed. This leads to an increase in the number of $\{\varyb[\tree,\node,\example,]\}_{\tree \in \forest,~\node \in \leaves[\tree],~\example \in \{1..\nexamples\}}$ variables, again increasing the size of the search space. However, this also provides more information regarding the values of the attributes of the reconstructed examples, partly explaining why the reconstruction error remains small, even if the number of possible reconstructions increases significantly.

\paragraph{Generalization Error.} Another interesting observation is that the generalization error does not affect our reconstruction approach. This was expected: what matters for reconstruction is the information provided by the forest regarding the training data (\emph{i.e.}, the number of samples passing through each branch fulfilling certain split conditions). Even a badly performing random forest can lead to accurate reconstructions if it encodes enough diverse information regarding its training data within the trees (regardless of unseen test data). 
This was already visible in all our experiments, where no correlation could be drawn between a random forest's generalization error and the success of our reconstruction attack.
More precisely, we investigated for a possible correlation between the reconstruction error and the generalization error (or train error or test error), but none of these analyses led to any visible trend. Finally, as it stands, having good or bad generalization capabilities does not appear to be a prerequisite for the success of our reconstruction approach.

\begin{table}[]
    \centering
   \begin{tabular}{@{}cccccccc@{}}
\toprule
\multirow{2}{*}{\textbf{\#Examples} $\nexamples$} & \multicolumn{2}{c}{\textbf{RF Accuracy}} & \textbf{Reconstruction error} & \multicolumn{4}{c}{\textbf{Reconstruction Time (s)}}      \\ \cline{2-8} 
                                     & \textbf{Train}      & \textbf{Test}      & \textbf{Avg}                  & \textbf{Avg} & \textbf{Std} & \textbf{Min} & \textbf{Max} \\ \hline
25                                   & 0.896               & 0.559              & 0.0                           & 5.8          & 1.1          & 4.1          & 7.4          \\ \midrule
50                                   & 0.860               & 0.556              & 0.0                           & 30.5         & 5.8          & 26.6         & 42.0         \\ \midrule
100                                  & 0.800               & 0.582              & 0.0                           & 84.1         & 12.0         & 65.8         & 99.1         \\ \midrule
200                                  & 0.770               & 0.617              & 0.0                           & 260.9        & 28.6         & 231.2        & 300.2        \\ \midrule
300                                  & 0.759               & 0.629              & 0.0                           & 467.9        & 80.2         & 386.1        & 621.9        \\ \midrule
400                                  & 0.741               & 0.632              & 0.0                           & 699.9        & 52.5         & 602.3        & 754.1        \\ \midrule
500                                  & 0.734               & 0.639              & 0.0                           & 1071.9       & 197.0        & 883.9        & 1448.7       \\ \midrule
750                                  & 0.725               & 0.643              & 0.0                           & 2219.4       & 428.4        & 1711.9       & 2818.9       \\ \midrule
1000                                 & 0.714               & 0.642              & 0.0                           & 3678.8       & 231.0        & 3300.8       & 4005.0       \\ \midrule
1500                                 & 0.704               & 0.650              & 0.0                           & 7362.2       & 1097.0       & 6485.9       & 9519.5       \\ \bottomrule
\end{tabular}
    \caption{Experiments (non-bagging case) on the reconstruction method's scalability, COMPAS dataset. Applying a simple power law regression, we observe that running times are in $\Theta(\nexamples^{1.7})$.}
    \label{tab:scalability_compas}
\end{table}

\begin{table}[]
    \centering
\begin{tabular}{@{}cccccccc@{}}
\toprule
\multirow{2}{*}{\textbf{\#Examples} $\nexamples$} & \multicolumn{2}{c}{\textbf{RF Accuracy}} & \textbf{Reconstruction error} & \multicolumn{4}{c}{\textbf{Reconstruction Time (s)}}      \\ \cmidrule(l){2-8} 
                                     & \textbf{Train}      & \textbf{Test}      & \textbf{Avg}                  & \textbf{Avg} & \textbf{Std} & \textbf{Min} & \textbf{Max} \\ \midrule
25                                   & 0.952               & 0.709              & 0.0                           & 9.1          & 2.9          & 6.2          & 14.6         \\ \midrule
50                                   & 0.964               & 0.748              & 0.0                           & 37.0         & 4.4          & 29.2         & 42.8         \\ \midrule
100                                  & 0.964               & 0.770              & 0.0                           & 188.2        & 58.0         & 117.4        & 278.3        \\ \midrule
200                                  & 0.949               & 0.767              & 0.0                           & 631.7        & 110.8        & 513.2        & 838.6        \\ \midrule
300                                  & 0.935               & 0.771              & 0.1                           & 4860.4       & 2457.7       & 1430.2       & 7695.8       \\ \midrule
400                                  & 0.925               & 0.778              & 0.1                           & 6119.3       & 2365.6       & 2304.4       & 8054.8       \\ \midrule
500                                  & 0.913               & 0.779              & 0.1                           & 6523.6       & 1558.4       & 4695.4       & 8429.2       \\ \midrule
750                                  & 0.905               & 0.780              & 0.0                           & 13911.5      & 5146.7       & 8764.9       & 19058.2      \\ \bottomrule
\end{tabular}
    \caption{Experiments (non-bagging case) on the reconstruction method's scalability, UCI Adult Income dataset. Applying a simple power law regression, we observe that running times are in $\Theta(\nexamples^{1.4})$.}
    \label{tab:scalability_adult}
\end{table}

\begin{table}[]
    \centering
    \begin{tabular}{@{}cccccccc@{}}
\toprule
\multirow{2}{*}{\textbf{\#Examples} $\nexamples$} & \multicolumn{2}{c}{\textbf{RF Accuracy}} & \textbf{Reconstruction error} & \multicolumn{4}{c}{\textbf{Reconstruction Time (s)}}      \\ \cmidrule(l){2-8} 
                                     & \textbf{Train}      & \textbf{Test}      & \textbf{Avg}                  & \textbf{Avg} & \textbf{Std} & \textbf{Min} & \textbf{Max} \\ \midrule
25                                   & 0.968               & 0.733              & 0.0                           & 6.7          & 1.5          & 4.0          & 8.1          \\ \midrule
50                                   & 0.976               & 0.723              & 0.0                           & 42.1         & 3.5          & 37.8         & 47.3         \\ \midrule
100                                  & 0.972               & 0.740              & 0.0                           & 148.3        & 16.6         & 127.6        & 169.4        \\ \midrule
200                                  & 0.965               & 0.751              & 0.0                           & 905.6        & 528.1        & 584.1        & 1958.1       \\ \midrule
300                                  & 0.957               & 0.763              & 0.0                           & 2369.8       & 1167.5       & 1316.7       & 4524.4       \\ \midrule
400                                  & 0.952               & 0.760              & 0.0                           & 2733.2       & 581.5        & 2065.6       & 3706.5       \\ \midrule
500                                  & 0.947               & 0.765              & 0.0                           & 6119.9       & 1880.1       & 3902.8       & 8671.4       \\ \bottomrule
\end{tabular}
    \caption{Experiments (non-bagging case) on the reconstruction method's scalability, Default of Credit Card Client dataset. Applying a simple power law regression, we observe that running times are in $\Theta(\nexamples^{2.4})$.}
    \label{tab:scalability_default_credit}
\end{table}

\section{Additional Experiments on Partial Reconstruction}\label{appendix:partial_reconstr}

In this appendix section, we perform complementary experiments on partial dataset reconstruction. More precisely, we consider the scenario where part of the training set attributes are known (for each training example). 
This scenario corresponds to the case where some of the attributes are publicly known, and the adversary's objective is only to retrieve the unknown (private) ones. As discussed in Section~\ref{sec:rel_works}, this setup corresponds to most of the reconstruction attacks found in the literature, where many works only attempt to reconstruct a single private attribute with knowledge of all the remaining ones~\citep{DBLP:conf/pods/DinurN03,doi:10.1146/annurev-statistics-060116-054123}.

For each of the three datasets considered in our experiments (introduced in Section~\ref{subsec:expes_setup}), we vary the number of known attributes between 0 and $\nattributes - 1$. The former case corresponds to the setup studied in Section~\ref{sec:experiments} (in which the adversary reconstructs the whole dataset), while in the latter case, only one attribute is unknown.
Between these two situations, our objective is also to characterize whether the knowledge of a number of attributes helps reconstruct the others, and to what extent.
Note that because binary attributes that are a one-hot encoding of the same original feature are not independent from each other, knowledge of one of them can fix the value of the others, which could bias the reconstruction results. For this reason, we consider each set of binary attributes that one-hot encode the same original feature as a single one. Then, the COMPAS dataset has $7$ such original features, the UCI Adult Income dataset has $14$, and the Default of Credit Card Client dataset has $16$. For each $\nattributes' \in \{0..\nattributes - \sum_{\ohegroup \in \ohevects}{(\lvert \ohegroup \rvert -1})\}$, we randomly pick $\nattributes'$ original attributes (\emph{i.e.,} either a binary attribute or a group of binary attributes one-hot encoding the same feature) that we assume are known. For such known attributes, their values for all the training set examples are fixed in the CP model introduced in Section~\ref{sec:method_implementation}. In other words, for each known attribute $\feature$, we assign the corresponding variables $\varx[\example,\feature]$ ($\forall \example \in \{1..\nexamples\}$) to their true value. 
The solver's task is then to find the value of the other attributes only.

To evaluate the proposed reconstruction, we first perform the examples' matching (with the actual training set) as described in Section~\ref{subsec:expes_setup}, using all the attributes. Then, the resulting reconstruction error is measured only on the unknown attributes. Note that performing the matching only using the unknown attributes would (artificially) result in lower reconstruction error rates, but would not make sense, as it would only evaluate whether the correct values for the unknown features are found (and not whether they are assigned to the correct example as indicated by the known attributes).
For these experiments, we focus on \sklearn{}'s default configuration (\emph{i.e.,} $\lvert \forest \rvert = 100$ trees and no maximum depth constraint). Moreover, we restrict our attention to the general case where bagging is used, as the (simpler) case without bagging is already successfully handled even without knowledge of any attribute. The experimental parameters are as described in Section~\ref{subsec:expes_setup}, and in particular, each run is averaged over five different random seeds. Finally, as already observed in Section~\ref{sec:experiments}, in a few experiments, the solver does not find any feasible reconstruction within the given time frame. 
Thus, we removed the experiments for which less than three runs were completed. This occurs in two cases, \emph{i.e.,} on the Default of Credit Card Client dataset, when the number of fixed attributes is at most $2$.

The reconstruction error (measured on the unknown attributes as aforementioned) is reported in Figure~\ref{fig:results_partial_reconstr_attrs} for all three datasets.
The results consistently show that knowledge of some attributes helps reconstructing the others. Moreover, the more attributes are known, the lower the error on the remaining (unknown) ones.
This suggests that considering the scenarios commonly used in the reconstruction literature only improves the results of our attack, as it successfully leverages knowledge of part of the dataset attributes.

It is worth noting that we also performed partial reconstruction experiments (not reported here) in which part of the training set examples (rather than attributes) are known by the attacker. Interestingly, we observed a different trend as knowledge of some examples did not really improve the reconstruction error for the others. A possible explanation for that (related to our findings of Section~\ref{appendix:benchmark}) lies in the Differential Privacy (DP) protection intrinsically offered by bagging~\citep{DBLP:conf/ijcai/LiuJG21}. Indeed, DP ensures that the trained forest does not depend too strongly on any single example, hence protecting each individual row within the training set. On the contrary, it does not directly protect the training set columns.

\begin{figure}[h!]
  \centering
    \begin{subfigure}[t]{0.48\textwidth}
    \centering\includegraphics[width=\textwidth]{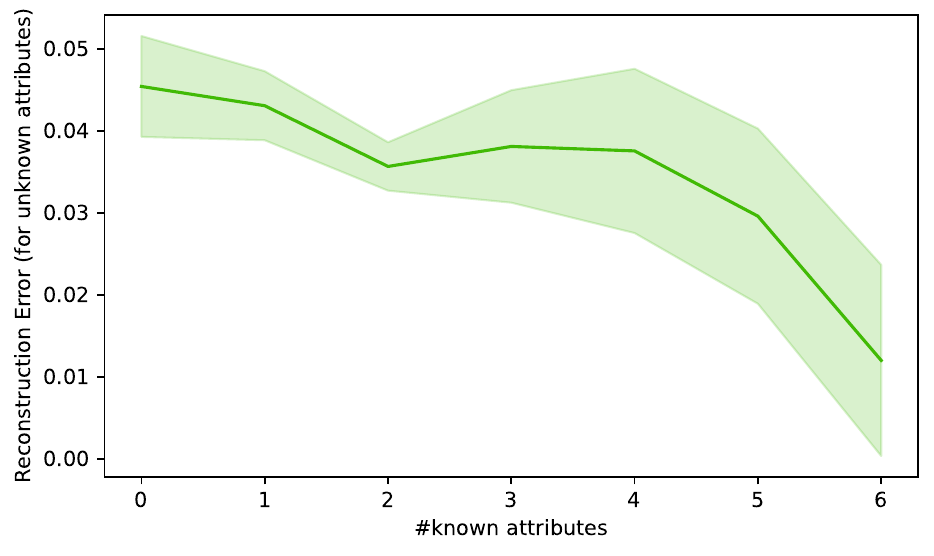}
    \caption{COMPAS dataset} 
    \label{fig:results_partial_reconstr_attrs_compas}
  \end{subfigure}
      \begin{subfigure}[t]{0.48\textwidth}
    \centering\includegraphics[width=\textwidth]{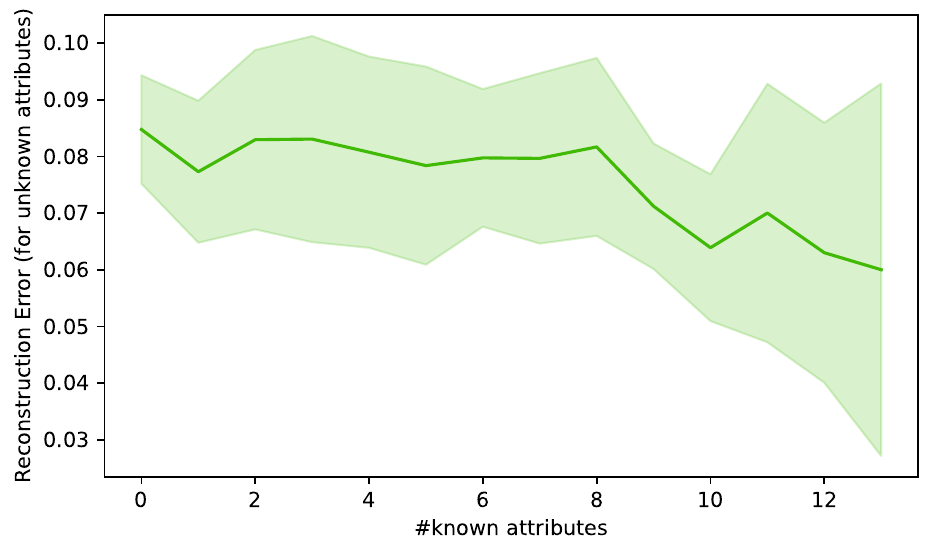}
    \caption{UCI Adult Income dataset}\label{fig:results_partial_reconstr_attrs_adult}
  \end{subfigure}

    \vspace{10pt}
  
    \begin{subfigure}[t]{0.48\textwidth}
    \centering\includegraphics[width=\textwidth]{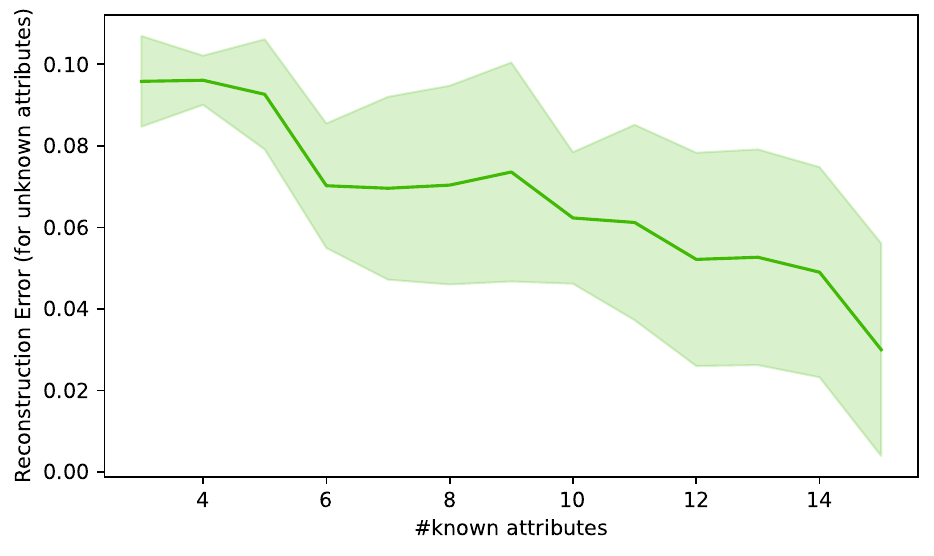}
    \caption{Default of Credit Card Client dataset}\label{fig:results_partial_reconstr_attrs_default_credit}
  \end{subfigure}

\caption{Results of reconstruction experiments  with knowledge of some of the attributes. We report the reconstruction error (for the unknown attributes) as a function of the number of known attributes in the forest's training set. For these experiments, all forests are learnt using \sklearn{}'s default configuration (\emph{i.e.,} $\lvert \forest \rvert = 100$ and no maximum depth constraint). Reconstruction errors are averaged over 5 different random seeds and we also report the standard deviation.}\label{fig:results_partial_reconstr_attrs}
\end{figure}

\end{document}